\documentclass[letterpaper]{article} 
\usepackage{aaai2026}  
\usepackage{times}  
\usepackage{helvet}  
\usepackage{courier}  
\usepackage[hyphens]{url}  
\usepackage{graphicx} 
\usepackage{natbib}  
\usepackage{caption} 
\usepackage{algorithm}
\usepackage{algorithmic}

\usepackage{newfloat}
\usepackage{listings}
\DeclareCaptionStyle{ruled}{labelfont=normalfont,labelsep=colon,strut=off} 
\floatstyle{ruled}
\newfloat{listing}{tb}{lst}{}
\floatname{listing}{Listing}
\title{Massively Parallel Proof-Number Search for Impartial Games and Beyond}
\author{
    Tomáš Čížek\textsuperscript{\rm 1},
    Martin Balko\textsuperscript{\rm 1},
    Martin Schmid\textsuperscript{\rm 1,\rm 2}
}
\affiliations{
    \textsuperscript{\rm 1}Department of Applied Mathematics, Charles University, Prague, Czech Republic,\\
    \textsuperscript{\rm 2}EquiLibre Technologies, Inc.\\
    \{cizek, balko, schmid\}@kam.mff.cuni.cz
}

\usepackage{booktabs}
\usepackage{makecell}
\usepackage{amssymb}
\usepackage{amsmath}
\usepackage{amsthm}
\usepackage{xcolor}
\usepackage{csquotes}
\usepackage{subcaption}
\usepackage{colortbl}
\usepackage{multirow}
\newtheorem{thm}{Theorem}
\definecolor{conwColor}{HTML}{ff666b}
\definecolor{mollColor}{HTML}{ff666b}
\definecolor{ajsColor}{HTML}{f89e54}
\definecolor{purrColor}{HTML}{f89e54}
\definecolor{glopColor}{HTML}{6db8c5}
\definecolor{newGlopColor}{HTML}{6db8c5}
\definecolor{spotsColor}{HTML}{4b8ce7}
\definecolor{pSpotsColor}{HTML}{4b8ce7}
\definecolor{maxcolor}{HTML}{E83F3A}
\newcommand{\conw}[3]{
    \cellcolor{conwColor!40}{#1}
    &
    \cellcolor{conwColor!40}{#2}
    &
    \cellcolor{conwColor!40}C
    &
    \cellcolor{conwColor!40}{#3}
}
\newcommand{\moll}[3]{
    \cellcolor{mollColor!25}{#1}
    &
    \cellcolor{mollColor!25}{#2}
    &
    \cellcolor{mollColor!25}M
    &
    \cellcolor{mollColor!25}{#3}
}
\newcommand{\ajs}[3]{
    \cellcolor{ajsColor!40}{#1}
    &
    \cellcolor{ajsColor!40}{#2}
    &
    \cellcolor{ajsColor!40}A
    &
    \cellcolor{ajsColor!40}{#3}
}
\newcommand{\purr}[3]{
    \cellcolor{purrColor!25}{#1}
    &
    \cellcolor{purrColor!25}{#2}
    &
    \cellcolor{purrColor!25}P
    &
    \cellcolor{purrColor!25}{#3}
}
\newcommand{\glop}[3]{
    \cellcolor{glopColor!25}{#1}
    &
    \cellcolor{glopColor!25}{#2}
    &
    \cellcolor{glopColor!25}G
    &
    \cellcolor{glopColor!25}{#3}
}
\newcommand{\glopMid}[3]{
    \cellcolor{newGlopColor!50}{#1}
    &
    \cellcolor{newGlopColor!50}{#2}
    &
    \cellcolor{newGlopColor!50}G
    &
    \cellcolor{newGlopColor!50}{#3}
}
\newcommand{\glopNew}[3]{
    \cellcolor{newGlopColor!50}{#1}
    &
    \cellcolor{newGlopColor!50}{#2}
    &
    \cellcolor{newGlopColor!50}G
    &
    \cellcolor{newGlopColor!50}{#3}
}
\newcommand{\spots}[3]{
    \cellcolor{spotsColor!25}{#1}
    &
    \cellcolor{spotsColor!25}{#2}
    &
    \cellcolor{spotsColor!25}S
    &
    \cellcolor{spotsColor!25}{#3}
}
\newcommand{\parallelspots}[3]{
    \cellcolor{pSpotsColor!40}{#1}
    &
    \cellcolor{pSpotsColor!40}{#2}
    &
    \cellcolor{pSpotsColor!40}S
    &
    \cellcolor{pSpotsColor!40}{#3}
}
\newcommand{\none}[1]{
    \cellcolor{white}{#1}
    &
    \multicolumn{3}{c|}{---}
}
\newcommand{\noneBorder}[1]{
    \cellcolor{white}{#1}
    &
    \multicolumn{3}{c|}{---}
}
\usepackage{tikz}
\newcommand{\circledchar}[2]{%
  \raisebox{2pt}{
    \tikz[baseline={(base)}]{
      \def\radius{1.9mm}
      \path[draw=black, fill=#1, line width=0.5pt] (0,0) circle (\radius);
      \node[text=white, font=\rmfamily\bfseries\small, align=center] (base) at (0,0) {#2};
    }%
  }%
}
\newcommand{\halfcirclechar}[3]{%
  \tikz[baseline={(base)}]{
    \def\radius{1.9mm}
    \path[draw=black, fill=#1, line width=0.5pt] (0,0) circle (\radius);
    \path[fill=#2, draw=none]
      (0,\radius) arc[start angle=90,end angle=-90,radius=\radius] -- (0,-\radius) -- cycle;
    \path[draw=black, line width=0.5pt] (0,0) circle (\radius);
    \node[text=white, font=\rmfamily\bfseries\small, align=center] (base) at (0,-0.1) {#3};
  }
}
\newcommand{\sem}[1]{\scriptsize\textcolor{gray}{\,$\pm$ #1}}

\begin{document}

\maketitle

\begin{abstract}
Proof-Number Search is a best-first search algorithm with many successful applications, especially in game solving.
As large-scale computing clusters become increasingly accessible, parallelization is a natural way to accelerate computation.
However, existing parallel versions of Proof-Number Search are known to scale poorly on many CPU cores.
Using two parallelized levels and shared information among workers, we present the first massively parallel version of Proof-Number Search that scales efficiently even on a large number of CPUs.
We apply our solver, enhanced with Grundy numbers for reducing game trees of impartial games, to the Sprouts game, a case study motivated by the long-standing Sprouts Conjecture.
Our algorithm achieves 332.9$\times$ speedup on 1024 cores, significantly improving previous parallelizations and outperforming the state-of-the-art Sprouts solver GLOP by four orders of magnitude in runtime while generating proofs 1,000$\times$ more complex.
Despite exponential growth in game tree size, our solver verified the Sprouts Conjecture for 42 new positions, nearly doubling the number of known outcomes.
\end{abstract}

\begin{links}
    \link{Code}{https://github.com/cizektom/spots}
\end{links}

\section{Introduction}

Game-solving is a well-established and difficult task in artificial intelligence, which involves computing outcomes under perfect play of all players, often requiring a deep traversal of vast game trees. 
Despite these challenges, several classical games have already been solved, including Connect Four~\cite{allis1988connect4}, Gomoku~\cite{allis1996gomoku}, Checkers~\cite{schaeffer2007checkers}, and Othello~\cite{takizawa2024othello}.
\emph{Proof-Number Search} \cite{allis1994} is among the most successful game-solving algorithms.
It is a best-first tree search algorithm designed to compute game outcomes efficiently.
This algorithm has been widely used due to its ability to focus on the most promising parts of the game tree, making it particularly effective in domains with large branching factors or unbalanced trees.
Beyond games, its variants and related algorithms have been applied in various other settings, for example, in chemistry \cite{franz2022}, graphical models \cite{dechter2007}, medicine \cite{heifets2012}, and theorem proving \cite{lample2022}.

With increasing computational power, it is natural to pursue parallel versions of Proof-Number Search; yet, despite many efforts, existing variants have failed to scale efficiently across many CPUs.
Several studies have suggested using parallel Proof-Number Search or its variants in diverse domains \cite{franz2022,kishimoto2015}, but scalability remains a challenge.
These limitations led \citet{kishimoto2012gametree} to pose the problem of developing a well-scaling, massively parallel Proof-Number Search.

Many combinatorial games remain far from being solved, including Chess, Shogi, Go, Cram, or Sprouts.
In \emph{Sprouts}, two players alternately connect $n$ given spots in the plane according to simple rules, where the last player unable to make a move loses; see~\citet{cizek2021implementation} for the implementation.
We use this game as our experimental domain, as it features highly unbalanced game trees with large branching factors, making it well-suited for Proof-Number Search.
Designed by Conway and Paterson to resist computer analysis~\cite{roberts2015genius}, Sprouts poses a challenging benchmark: the complexity of its game tree surpasses Chess and Go for relatively small values of~$n$; see Figure~\ref{fig_tree_complexity}.
The computation of outcomes of Sprouts positions is further motivated by the long-standing \emph{Sprouts Conjecture} \cite{applegate1991computer}, which states that the \emph{$n$-spot position}, consisting of $n$ given spots, is winning for the first player if and only if $n$ is congruent to 3, 4, or 5 modulo 6.

\begin{figure}[!htb]
\centering
    \includegraphics[scale=0.87]{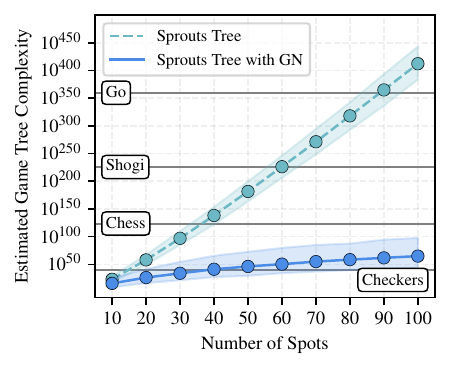}
\caption{Estimated game tree complexity of Sprouts (with and without Grundy numbers) compared to other games.}
\centering
\label{fig_tree_complexity}
\end{figure}

\subsection{Our Contributions}

We address the problem posed by \citet{kishimoto2012gametree} and develop a new variant of Proof-Number Search that scales efficiently on many CPUs. 
Building on this, we implement a parallel solver for combinatorial games.
Together with a new memory-efficient algorithm for impartial games, we achieve numerous new results for the game Sprouts.

\paragraph{Well-Scaling Parallel Proof-Number Search.}
We intro\-duce \emph{PNS-PDFPN}, a massively parallel Proof-Number Sea\-rch for distributed-memory systems combining two parallelized levels with shared information among workers. It achieves 208.6$\times$ speedup on 1024 cores over sequential DFPN, significantly outperforming prior methods, whose best scaling factor in the same setting was 21.3.
Using minimal domain knowledge, the speedup increases to 332.9$\times$.

\paragraph{DFPN Search with Grundy Numbers.}
Proof-Number Search is designed to solve problems that are represented as AND/OR trees.
We formalize the notion of game trees with Grundy numbers to reduce the complexity of game trees of the so-called impartial games, as introduced by \citet{lemoineViennot12Inevitable}.
We adapt \emph{Depth-First Proof-Number Search with Grundy numbers}, a popular and memory-efficient variant of Proof-Number Search, to work with such trees.

\paragraph{Solving Sprouts Positions.}
To demonstrate the efficiency of our new algorithms, we implemented a solver for the Sprouts game, and used it to determine many new outcomes.
Our solver outperforms the state-of-the-art solver \emph{GLOP} by \citet{lemoine2015nimber}, achieving a four-order-of-magnitude speedup and generating proofs 1,000 times more complex.
Despite the exponential growth of the game trees, our solver verifies the Sprouts Conjecture for 42 new positions, nearly doubling the number of known outcomes.

\subsection{Related Work}

A considerable amount of work has been devoted to designing parallel versions of Proof-Number Search. These include \emph{RP-PNS}, introduced by~\citet{saito2010parallel}, a parallel variant of DFPN by \citet{kaneko2010parallel}, and \emph{SPDFPN} algorithm by \citet{pawlewicz2013scalable}. All of these were designed for shared-memory systems, which typically offer only a limited number of CPU cores.
To leverage the greater computational power of clusters with many cores, algorithms tailored to distributed-memory systems are required. Such algorithms include \emph{ParaPDS} by \citet{kishimoto1999parallel}, \emph{Job-Level Proof-Number Search} by \citet{wu2010joblevel, wu2013joblevel}, and \emph{PPN\textsuperscript{2}} search by \citet{saffidine2012parallel}.
A distributed-memory parallel Proof-Number Search was also used by \citet{schaeffer2005solving,schaeffer2007checkers} in solving Checkers.

As noted by \citet{kishimoto2012gametree}, the scalability of existing algorithms was evaluated only on a relatively small number of cores, and, despite various successes in game solving, their efficiency declined rapidly as core counts increased. 
While parallel DFPN by \citet{kaneko2010parallel} and SPDFPN have speedups of 3.58 on 8 cores and 11.8 on 16 cores, respectively, they are limited to shared-memory systems.
The best scaling factors for distributed-memory systems are reported for PPN\textsuperscript{2} search, which achieved speedups of 7.2, 11.1, and 18.3 on 16, 32, and 64 cores, respectively.

Determining the outcomes of $n$-spot Sprouts positions has a long history dating back to the 1960s.
Conway himself determined the outcomes for $n \leq 3$, and later Mollison extended the results to
$n=6$ by hand~\cite{gardner1967mathematical}.
The first computer-based solver was developed by \citet{applegate1991computer}, who used it to prove outcomes for $n\leq 11$.
In 2006, Purinton created the program \emph{AuntBeast}, which solved positions up to 14 spots.
A major advance was made by the state-of-the-art solver GLOP \cite{lemoine2015nimber} who used Alpha-Beta Pruning combined with Grundy number theory to solve all positions for 
$n\leq 32$.
These results were later extended to all values of
$n \leq 44$ and selected cases up to $n=53$, by incorporating a basic variant of Proof-Number Search combined with Grundy numbers.
All known outcomes agree with the Sprouts Conjecture.

\section{Preliminaries}

We now briefly review relevant preliminaries from combinatorial game theory; for more details, see \citet{kishimoto2012gametree} and \citet{berlekamp2003winning}.
A \emph{combinatorial game} is a two-player, deterministic game with perfect information that ends in a finite number of moves.
Under \emph{normal play convention}, the last player able to move wins.
A combinatorial game is \emph{impartial} if the available moves from any given position are the same for both players, regardless of whose turn it is.  
Combinatorial games include Chess, Checkers, Go, Othello, and Shogi, whereas impartial games include Nim, Sprouts, Quarto, and Cram.

\subsection{Game Trees in Negamax Form}

Game trees of combinatorial games can be represented by the so-called \emph{AND/OR trees}.
Here, we instead work with AND/OR trees in negamax fashion using \emph{NAND trees}, which are easier to work with, as they eliminate the unnecessary distinction between the AND and OR nodes in subsequent definitions.
The nodes at odd and even levels of a NAND tree correspond to positions where the first and second player, respectively, is on the move.
In particular, the root corresponds to the initial position with the first player on the move.
A \emph{terminal} is a node with no children, which represents a position where the game ends.
In a partially expanded NAND tree, a non-terminal node is \emph{internal} if some of its children are generated, and a \emph{leaf} if none of them are.  

Each node is associated with one of the following \emph{values}: win, loss, or unknown.
The value of a node is \emph{unknown} if the outcome of the associated position is not implied by the currently expanded subtree of the node.
Otherwise, it is \emph{win} or \emph{loss}, depending on whether the player on the move at the associated position has a winning strategy or not.
Thus, the value of every terminal is loss in any combinatorial game under the normal play convention.
The value of an internal node is win if at least one of its children is loss, and it is loss if all its children are wins; otherwise, the value is unknown.
The value of a leaf is always unknown.
A node is \emph{proved} or \emph{disproved} if its value equals to win or loss, respectively.

Given a position $P$, the goal is to find a \emph{proof} of $P$, which is a NAND tree rooted in $P$ whose value is known.
These proofs correspond to the so-called \emph{weak solutions}, as they give the outcome of $P$ and a winning strategy for one of the players.
In contrast, \emph{strong solutions} provide outcomes and winning strategies for all positions reachable in the game.

\subsection{Proof-Number Search}

We first describe a basic variant \emph{PNS} of the Proof-Number Search for NAND trees.
In each step, PNS selects and expands a \emph{most proving node} (\emph{MPN})---a leaf in a currently expanded tree whose solution would contribute the most to the solution of the root.
To find an MPN, each node $v$ in a NAND tree is associated with a \emph{proof number} $pn(v) \in \mathbb{N}_0 \cup \{ \infty \}$ and a \emph{disproof number} $dn(v) \in \mathbb{N}_0 \cup \{ \infty \}$, representing the lower bounds on the minimum number of leaves in the subtree of $v$ that have to be solved to prove or disprove $v$, respectively.
Both these values are computed in a bottom-up manner.
For leaves, they are initialized as $pn(v) = dn(v) = 1$.
For terminals in games under the normal play convention, $pn(v) = \infty$ and $dn(v) = 0$.
The values $pn(v)$ and $dn(v)$ of an internal node $v$ are defined as $pn(v) = \min_{c}{dn(c)}$ and $dn(v) = \sum_{c}{pn(c)}$, where the minimum and the sum are taken over the children $c$ of~$v$.
At the start, the algorithm initializes $pn(r)$ and $dn(r)$ of the root $r$.
At each step, it then finds an MPN by descending from the root and always selecting the child with the lowest disproof number with ties broken arbitrarily.
PNS then expands the selected MPN by generating all its children $c$ and initializing their $pn(c)$ and $dn(c)$.
At the end of each step, it updates the proof and disproof numbers on the path back to the root.

\paragraph{Depth-First Proof-Number Search.}
One drawback of PNS is that it stores the entire expanded tree in memory, which is usually consumed very quickly.
Therefore, \citet{nagai1999phdthesis} introduced \emph{Depth-first Proof-Number Search} (\emph{DFPN}), a space-efficient
variant of PNS.

DFPN selects the MPN in a depth-first search manner.
Let $\mathcal{P}$ be a currently explored path consisting of internal nodes leading to the MPN.
The key idea of DFPN allows one to store only the nodes along $\mathcal{P}$ together with their children while preserving the properties of PNS and staying deep in the tree as long as the MPN is guaranteed to occur there.
To achieve that, for each node $v$ of $\mathcal{P}$, DFPN maintains two thresholds $pt(v) \in \mathbb{N}_0 \cup \{\infty\}$ and $dt(v) \in \mathbb{N}_0 \cup \{\infty\}$ for its proof and disproof numbers.
The MPN occurs in the subtree of $v$ if and only if $pn(v) < pt(v)$ and $dn(v) < dt(v)$, where $\infty < \infty$ is interpreted as false.
Otherwise, DFPN backtracks along $\mathcal{P}$ until it reaches $v'$ that satisfies $pn(v') < pt(v')$ and $dn(v') < dt(v')$, updating the proof and disproof numbers along $\mathcal{P}$.
DFPN then resumes MPN selection from $v'$.

First, set $pt(r)=dt(r)=\infty$ for the root $r$.
Let $v$ be the currently visited node, $w$ and $w'$ be its children with the lowest and second lowest disproof number, where $w$ is the next node to be selected.
The thresholds of $w$ are then set as
\begin{align}
\begin{split}
\label{align_dfpn_3}
    pt(w) &= dt(v) - dn(v) + pn(w),\\
    dt(w) &= \min \{pt(v), dn(w') + 1\}.
\end{split}
\end{align}
The decreased memory consumption of DFPN comes at the cost of increased time complexity since the proof and disproof numbers of some of the previously visited nodes are lost and potentially need to be recomputed.
To balance this trade-off, DFPN is often combined with a transposition table that maintains previously computed values.

\subsection{Grundy Numbers}

As shown by \citet{lemoineViennot12Inevitable} and by \citet{belingRogalski2020}, the \emph{Spra\-gue--Grundy Theorem} \cite{grundy1939mathematics, sprague1935uber} can be applied to effectively simplify game trees of impartial games under the normal play convention; see Figure~\ref{fig_tree_complexity}.
Since we build on these techniques, we briefly recall the necessary background.

\emph{Nim} is a game that is played on $h$ heaps with $n_1, \dots, n_h \in \mathbb{N}_0$ objects.
With each move, a player must remove at least one object from a single heap, and the first player with no move loses the game.
Nim is a strongly solved game, as~\citet{bouton1901nim} proved that the outcome of Nim is loss if and only if $n_1 \oplus \cdots \oplus n_h = 0$, where $\oplus$ is a bitwise exclusive or.

Let $P_1,\dots, P_k$ be positions of an impartial game under the normal play convention.
Then, the \emph{combination } of $P_1, \dots, P_k$, denoted by $P_1 + \dots + P_k$, is a position in which, in each turn, the players decide to move in one of the positions $P_1, \dots, P_k$ while leaving the other positions untouched.
The first player with no move in any of $P_1, \dots, P_k$ loses in the combination $P_1 + \dots + P_k$.
A position $P$ is \emph{atomic} if it cannot be expressed as a combination of at least two non-empty positions; otherwise, it is \emph{decomposable}, in which case $P=P_1+\dots+P_k$ for non-empty atomic positions $P_1,\dots,P_k$ and $k \geq 2$.
Two positions $P$ and $Q$ are \emph{equivalent} if, for any position $R$, the combinations $P + R$ and $Q + R$ have the same outcome.

The \emph{Grundy number} $gn(P)$ of a position $P$ (also called the \emph{nimber} of $P$) is defined recursively as follows.
If $P$ is a terminal position, then $gn(P) = 0$.
Otherwise, $gn(P)$ is equal to $\min{\mathbb{N}_0 \setminus G}$ where $G$ is the set of the Grundy numbers of the children of $P$.
Using $\ast n$ to denote a Nim position with a single heap of $n$ objects, we formulate the \emph{Sprague--Grundy Theorem}~\cite{grundy1939mathematics, sprague1935uber}, which states that each equivalence class of impartial games under the normal play convention can be represented by a unique Grundy number.

\begin{thm}[The Sprague--Grundy Theorem]
\label{thm-SpragueGrundy}
Each posi\-tion $P$ of an impa\-rtial game under the normal play convention is equivalent to $\ast gn(P)$.
\end{thm}

It follows from Theorem~\ref{thm-SpragueGrundy} that the outcome of $P_1 + \dots + P_k$ is loss if and only if $gn(P_1) \oplus \dots \oplus gn(P_k) = 0$.
Thus, we can compute the outcome of a decomposable position $P=P_1+\cdots+P_k$ more efficiently using the Grundy numbers $gn(P_1), \dots, gn(P_k)$ \cite{lemoineViennot12Inevitable}.

\section{Methods}

First, we formalize extended NAND trees with Grundy numbers that were implicitly used by~\citet{lemoineViennot12Inevitable} to efficiently determine the outcomes of decomposable positions.
Then, we describe and improve PNS for extended NAND trees by \citet{lemoine2011phdthesis} and introduce DFPN for such trees.
Finally, we describe our massively parallel version PNS-PDFPN of Proof Number Search. 

\subsection{NAND Trees with Grundy Numbers}

To apply the Sprague--Grundy Theorem, we need to determine the Grundy number $gn(P)$ of a position $P$ in an impartial game $G$ under the normal play convention.
We efficiently compute $gn(P)$ using the following recursive procedure introduced by~\citet{lemoineViennot12Inevitable}: start with $n=0$ and increment $n$ until the outcome of the combination $P+\ast n$, called a \emph{couple}, is loss.
Then $gn(P)=n$, because $gn(P) = n$ if and only if the outcome of $P+\ast n$ is loss. 

If $P$ is atomic, then the outcome of the couple $P+\ast n$ is obtained as before from the outcomes of its children, which are the couples of the form $P'+\ast n$ and $P + \ast n'$ where $P'$ is a child of $P$ and $0 \leq n' < n$.
The outcome of $P+\ast n$ with decomposable $P=P_1+\cdots+P_k$ is obtained by first computing $gn(P_1),\dots,gn(P_{k-1})$ using the above procedure and then by computing the outcome of the couple $P_k+\ast n'$, where $n' = n\oplus gn(P_1)\oplus\cdots\oplus gn(P_{k-1})$ and $P_k$ is now atomic.

The nodes of a NAND tree with Grundy numbers correspond either to atomic positions $Q$ of $G$ or to couples $P+\ast n$ consisting of a position $P$ of $G$ and a Nim position $\ast n$ with $n \in \mathbb{N}_0$.
We have three types of nodes: decomposable, atomic, and Grundy; see Figure~\ref{fig_nand_tree_nimbers}.
A node $v$ corresponding to a couple $P+\ast n$ is \emph{decomposable} if $P$ is decomposable and \emph{atomic} if $P$ is atomic.
Each node $u$ corresponding to an atomic $Q$ is a \emph{Grundy} node.
The children of~$v$ correspond to the children of $P+\ast n$ if $P$ is atomic and to Grundy nodes $P_1,\dots,P_{k-1}$ and the couple $P_k + \ast n'$ if $P$ is decomposable.
The children of~$u$ correspond to $Q+\ast 0, \dots, Q+\ast c(Q)$, where $c(Q)$ is the number of children of $Q$.

\begin{figure}[t]
\centering
\includegraphics[scale=0.77]{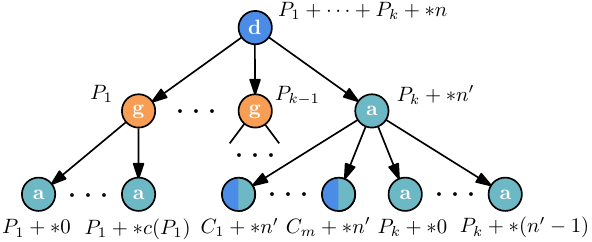}
\caption{A scheme of a NAND tree with Grundy numbers.
We use $C_1, \dots, C_m$ for the children of $P_k$ and \circledchar{spotsColor}{d}, \circledchar{glopColor}{a}, \circledchar{ajsColor}{g} for decomposable, atomic, and Grundy nodes, respectively, and \halfcirclechar{spotsColor}{glopColor}{} for nodes that can be atomic or decomposable.}
\centering
\label{fig_nand_tree_nimbers}
\end{figure}
The value of a node $v$ corresponding to $P+\ast n$ is again either win, loss, or unknown.
For terminals, which correspond to positions $\emptyset+\ast 0$, where $\emptyset$ is the empty position in $G$, the value is loss.
If $v$ is atomic, then its value is determined in the same way as in the original NAND tree with respect to the children of $P+\ast n$.
If $v$ is decomposable, then the value is unknown if $v$ is a leaf, and is equal to the value of the node corresponding to $P_k + \ast n'$ if $v$ is internal.
For a Grundy node $u$ corresponding to $Q$, the value of $u$ is either a Grundy number $gn(Q)$ or unknown.
The value of $u$ equals $n$ if there is a losing child $Q + \ast n$ of $Q$, and is unknown otherwise.

\subsection{Proof-Number Search with Grundy Numbers}
To extend PNS to NAND trees with Grundy numbers, we only need to adapt the proof and disproof numbers and specify the next node $w$ to be selected.
The resulting algorithm \emph{PNS with GN} then proceeds exactly as PNS.

We define the proof and disproof numbers so that they give lower bounds on the size of the proof of a given position, which guarantees that PNS finds a minimal proof.
For leaves, we initialize them to $1$.
For atomic nodes, the proof and disproof numbers and the next node $w$ are defined as in the original setting.
If $u$ is a Grundy node, then it corresponds to an atomic position~$Q$ and we use $v_i$ to denote the child of $u$ representing $Q+\ast i$.
The children of $u$ are atomic and are generated incrementally with increasing~$i$, each only after the previous one has been proved, as prescribed by the recursive procedure for the calculation of $gn(Q)$.
If $v_j$ is the last child of $u$ generated so far, then we set $pn(u) = dn(u) = \min{\{pn(v_j), dn(v_j)\}}$ and $w = v_j$.
Finally, let $v$ be a decomposable node.
Then $v$ corresponds to a couple $P+\ast n$ where $P=P_1+\cdots+P_k$ and $k \geq 2$. 
We use $u_1,\dots,u_k$ to denote the nodes where each $u_i$ is a Grundy node corresponding to $P_i$ and $v'_k$ is an atomic node corresponding to $P_k + \ast n'$.
Note that only $u_1,\dots,u_{k-1},v'_k$ are the children of $v$, and that in order to generate $v'_k$, we first need to know the Grundy numbers of $P_1,\dots,P_{k-1}$, as $n' = n\oplus gn(P_1)\oplus\cdots\oplus gn(P_{k-1})$.
Thus, if $v'_k$ has not been generated, we use $u_k$ instead and set $pn(v) = dn(v) = \sum_{i = 1}^k pn(u_i)$ and $w$ to be the first node from $u_1,\dots,u_{k-1}$ with unknown Grundy number.
Otherwise, we let $pn(v) = pn(v'_k)$, $dn(v) = dn(v'_k)$, and $w = v'_k$.

We improve this version of PNS from \citet{lemoine2011phdthesis} by following \citet{schijf1994transpositions} and merging nodes representing identical states to transform the tree into a directed acyclic graph.
This significantly reduces the computation time without increasing memory overhead when combined with storing all touched nodes.

\paragraph{DFPN with Grundy Numbers.}

We now introduce a new variant of DFPN for extended NAND trees with Grundy numbers (\emph{DFPN with GN}).
The algorithm follows the same structure as DFPN for NAND trees, requiring only modifications to the thresholds and the way they are computed.

For each node $v$ from the currently explored path $\mathcal{P}$, we again maintain thresholds $pt(v)$ and $dt(v)$ upper-bounding proof and disproof numbers.
However, this time we also need to maintain a new threshold $mt(v)\in \mathbb{N}_0 \cup \{\infty\}$ with parameters $ps(v)\in \mathbb{N}_0$ and $ds(v)\in \mathbb{N}_0$.
Intuitively, $mt(v)$ bounds the minimum of $pn(v)$ and $dn(v)$, shifted by $ps(v)$ and $ds(v)$, which is necessary due to decomposable nodes.

To simplify the definition of the thresholds, we replace each Grundy node $u$ in its decomposable parent with the node $v_j$, which would be selected as the next node if we were in $u$.
Then, we can keep the thresholds only for atomic and decomposable nodes.
Now, let $v$ be the currently visited node, and let $w$ be the next node to be selected. 
For an atomic $v$, we let $w'$ be its child with the second-lowest disproof number, and we set the thresholds as
\begin{align}
    \label{eq:gn_thresholds}
    \begin{split}
            pt(w) &= dt(v) - dn(v) + pn(w),\\
            dt(w) &= \min \{pt(v), dn(w') + 1\},\\
            mt(w) &= mt(v),
    \end{split}
\end{align}
where $ps(w) = ds(v) + dn(v) - pn(w)$ and $ds(w) = ps(v)$.
For a decomposable $v$, we set $pt(w)=pt(v)$, $dt(w)=dt(v)$, and $mt(w)=mt(v)$ if $v'_k$ has been generated.
If $v'_k$ has not been generated yet, we define the thresholds as
\begin{align*}
    \begin{split}
            pt(w) &= dt(w) = \infty,\\
            mt(w) &= t(v) - pn(v) + \min \{pn(w), dn(w) \},
    \end{split}
\end{align*}
where $ps(w) = ds(w) = 0$ and $t(v) = \min \{ pt(v), dt(v), \allowbreak mt(v) - \min \{ ps(v),ds(v) \} \}$.
For the root $r$, we initialize $pt(r)=dt(r)=mt(r) =\infty$ and $ps(r)=ds(r)=0$.
With this choice of thresholds, we obtain the following result.

\begin{thm}
\label{thm-dfpn}
For every node $v$ of $\mathcal{P}$, the MPN is in the subtree of $v$ if and only if $pn(v) < pt(v)$, $dn(v) < dt(v)$, and $\min \{pn(v) + ps(v), dn(v) + ds(v)\} < mt(v)$.
\end{thm}
See Appendix for the proof.
To balance the time complexity and memory consumption of DFPN with GN, we also incorporate a transposition table of proof and disproof numbers, maintained by the replacement strategy proposed by~\citet{nagai1999phdthesis}.
In addition, we store the Grundy numbers derived during computation in a separate database.
Since this database is much smaller and fits easily into memory, we do not apply any replacement strategy to it.

\subsection{PNS-PDFPN Algorithm}
Our new parallel variant \emph{PNS-PDFPN} of Proof-Num\-ber Search operates on two parallelized levels.
The first PNS level targets distributed-memory systems and is based on Job-Level Proof-Number Search used by \citet{schaeffer2007checkers}, \citet{saffidine2012parallel}, and \citet{wu2013joblevel}.
At this level, the \emph{master} process maintains the current proof state and repeatedly assigns jobs to \emph{workers} for asynchronous second-level processing.
At the second level, each worker performs the parallel DFPN algorithm by \citet{kaneko2010parallel}, which we call \emph{PDFPN}, to utilize shared memory within a single cluster node.
We also introduce the sharing of key results between workers to reduce search overhead.
We now describe both levels in detail.

\paragraph{First-Level Parallelization.}
The master and workers typically run on separate nodes of a cluster and communicate over an interconnected network.
When a worker becomes idle, the master selects a so-called pseudo-MPN leaf $\ell$ and assigns it to the worker.
The worker then processes $\ell$ asynchronously until it is solved or the maximum number of expansions is reached.
The resulting proof and disproof numbers of $\ell$ and its children are then sent back to the master, which uses them to expand $\ell$ in the master tree and updates its proof and disproof numbers.
During processing of~$\ell$, the worker periodically sends updated values $pn(\ell)$ and $dn(\ell)$ to keep the master tree as up-to-date as possible.
To prevent reassignment of the same jobs, the assigned leaves are locked, and the definition of proof and disproof numbers is slightly adjusted to avoid misleading attraction to the locked leaves; see \citet{wu2013joblevel} for details.

Similar to \citet{xichen15Hex} and their JL-UCT Search, our master maintains a database of key computed results shared with all workers to reduce search overhead.
Newly derived results are sent to workers along with their assigned jobs, while workers report new results back to the master with each update.
Finally, if multiple workers are located on the same node, then they are grouped together to share the same local version of the database.

\paragraph{Second-Level Parallelization.}
As cluster nodes consist of multiple CPU cores, it becomes advantageous, especially with an increasing number of workers, to start parallelizing the workers themselves, rather than adding more workers that would otherwise process less relevant jobs.
This worker reduction also decreases communication overhead and allows sharing of proof and disproof numbers within a node.

Our workers are based on DFPN to make full use of available memory; thus, we parallelize them using PDFPN.
In this algorithm, there are as many threads as the number of cores assigned to the worker.
Each thread performs an independent DFPN search in the subtree of $\ell$, while sharing the lock-protected transposition table of proof and disproof numbers.
To decrease redundant thread computation in shared subtrees, the disproof numbers are virtually increased during the selection of the next node $w$ by the number of threads currently computing in the corresponding subtrees.
The thresholds in \eqref{align_dfpn_3} and \eqref{eq:gn_thresholds} are then modified by subtracting $th(w)$ from the term $dn(w') + 1$, where $th(w)$ is the number of threads in the subtree of $w$.
Additionally, if a thread solves a position currently computed by a different one, the other thread is notified to backtrack to that position.

\paragraph{Enhancements for Impartial Games.}
To employ the theory of Grundy numbers, we incorporate PNS with GN and DFPN with GN.
We use all computed Grundy numbers as the key shared results, since they are highly reusable and inexpensive to share.
To enable parallel computation of multiple children of a decomposable node, the next node $w$ within a decomposable node is selected so that the value $\min\{pn(w), dn(w)\}$ is minimum, rather than always selecting the first unsolved child.

\section{Experiments}

We implemented \emph{SPOTS}, a PNS-PDFPN-based solver for combinatorial games, augmented with Grundy numbers to reduce game trees of impartial games effectively.
From now on, we use PNS and DFPN to refer to their Grundy-enhanced variants.
We evaluate the performance of SPOTS on Sprouts, an impartial game well-suited for Proof-Number Search due to its highly unbalanced game trees, large branching factors, and the fact that it often naturally decomposes into independent subpositions.
To incorporate Sprouts into SPOTS, we implemented the string-based position representation from \citet{applegate1991computer}.
Our experiments were carried out on the 29-spot and 47-spot positions, which are the largest that allow each trial to complete within 24 hours.
The experimental results are reported as the mean $\pm$ standard error of the mean, based on three independent measurements.

\subsection{Efficiency of DFPN with GN}

In Figure~\ref{fig_first_experiment}, we demonstrate how DFPN with GN allows for tuning the trade-off between computation time and memory consumption.
By adjusting the capacity~$C$ of the transposition table, memory requirements can be substantially reduced at the cost of a moderate increase in computation time.
We observe that the transposition table with $C = 10^6$ is already sufficiently large, as it is not fully saturated. Despite this, our PNS implementation still outperforms this DFPN configuration, as it operates faster on its explicitly expanded tree; however, the PNS tree can grow indefinitely. 

We also compare our PNS with the PNS implemented by \citet{lemoine2015nimber} in their solver GLOP.
Unlike our method, GLOP operates on a proper tree without transpositions and discards all visited nodes in the solved subtrees to reduce memory usage.
Along with a roughly three-times faster state representation, our design choices lead to a roughly 100-fold reduction in computation time compared to GLOP while maintaining a comparable maximum tree size.

\begin{figure}[t]
\centering
\includegraphics[scale=0.9]{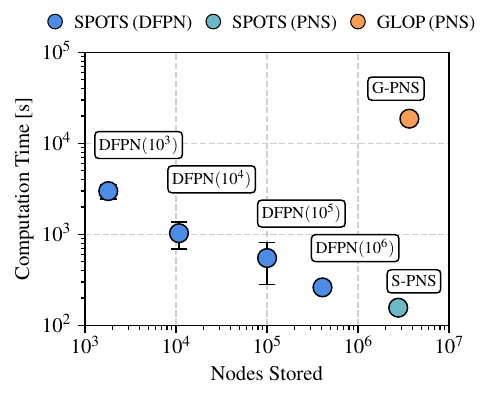}
\caption{Performance of PNS implementations in GLOP and SPOTS and of DFPN($C$) in SPOTS on the 29-spot position, where $C$ is the capacity of the transposition table.}
\label{fig_first_experiment}
\end{figure}

\subsection{Scaling Efficiency of PNS-PDFPN}

PNS-PDFPN includes five tunable parameters: \texttt{workers}, the total number of workers; \texttt{iterations}, the maximum number of expansions per job; \texttt{updates}, the number of iterations between each update sent to the master; \texttt{grouping}, the number of workers grouped within a single node; and \texttt{threads}, the number of threads assigned to PDFPN per worker.
The number of CPU cores used by PNS-PDFPN is then equal to \texttt{workers} times \texttt{threads} plus the number of cores allocated to the master.
We now analyze the effect of the worker synchronization and parameter setting on the scaling efficiency of the algorithm, which is reported with respect to the number of CPU cores allocated to workers, excluding the master CPU.

\paragraph{Retaining and Sharing Second-Level Information.}

The parallel \emph{PPN\textsuperscript{2}} search by \citet{saffidine2012parallel} corresponds to PNS-PNS using PNS at both levels.
However, after each job is completed, the second-level PNS search tree is discarded, losing the entire worker state. 
In contrast, our PNS-DFPN workers retain their transposition tables, completely preserving their state throughout the whole computation.
Figure~\ref{fig_second_experiment} shows how the retention and sharing of the worker state impact scaling efficiency.
Although PNS-PNS scales well to up to 64 cores, worker state resets hamper its performance so significantly that it is outperformed by sequential DFPN($10^6$).
However, when the key results are retained, specifically the computed Grundy numbers, the speedup improves significantly.
Further performance gains can be achieved by PNS-DFPN preserving computed proof and disproof numbers.
However, this has a smaller effect than retaining the Grundy numbers.
Finally, sharing Grundy numbers among workers yields a further significant performance gain.

\begin{figure}[t]
\centering
\includegraphics[scale=0.9]{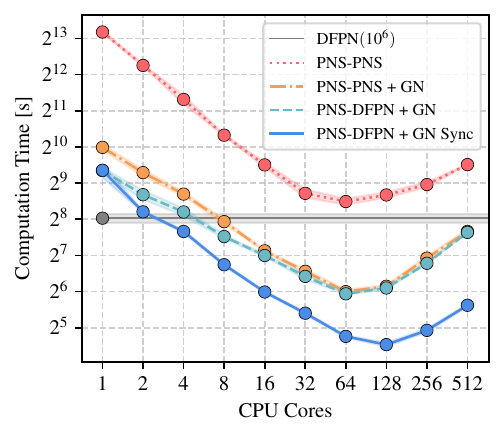}
\caption{Comparison of PNS-PNS and PNS-DFPN, including the impact of retention of Grundy numbers (GN) in workers and their synchronization (GN Sync). All variants are run on the 29-spot position with 100 \texttt{iterations}, 100 \texttt{updates}, no grouping, and no second-level parallelization.
}
\label{fig_second_experiment}
\end{figure}

\paragraph{PNS-PDFPN Ablations.}
We evaluate the impact of the PNS-PDFPN parameters on the scaling efficiency, measured relative to the runtime of DFPN($5 \cdot 10^7$), serving as a baseline for comparing various parallel Proof-Number Search algorithms.
We compare PNS-PDFPN against PDFPN by \citet{kaneko2010parallel} and the PNS-DFPN scheme employed by \citet{schaeffer2007checkers} to solve Checkers, which, according to our measurements, represent state-of-the-art approaches for shared- and distributed-memory systems, respectively.

In Figure~\ref{fig_third_experiment}, we incrementally add parameters until we arrive at the full PNS-PDFPN with all the improvements.
For each setup, we report the speedup achieved using the best-performing parameter setting, where PNS-DFPN is already optimized for \texttt{iterations} and \texttt{updates}, as proposed by \citet{saffidine2012parallel}.
Without GN synchronization, PNS-DFPN scales very poorly.
Enabling synchronization improves scalability up to 512 cores, but then the performance degrades due to significant synchronization overhead.
This issue is overcome by grouping workers, reducing the amount of shared results, and thus the synchronization overhead.
Further scalability is achieved by introducing second-level parallelization, which ultimately leads to a speedup of $208.67 \pm 9.17$ without relying on any domain knowledge, which exceeds even the best speedup of $21.35 \pm 2.40$ achieved by PDFPN using shared memory.

We observe a notable improvement in scaling efficiency after incorporating the child-ordering heuristic proposed by \citet{lemoine2015nimber} to break ties during the selection of the next node $w$.
Adding the heuristic does not enhance the overall performance if applied to optimally tuned parameters without the heuristic, as shown in the upper part of Table~\ref{tab:heuristicEvaluation}.
However, the heuristic becomes beneficial when the parameters are tuned specifically for its use; see the bottom part of Table~\ref{tab:heuristicEvaluation}.
The heuristic then enables more effective utilization of second-level parallelism in longer-running jobs by providing better guidance to branches that are more likely to yield shorter proofs.
With this modest domain knowledge, the final speedup of PNS-PDFPN reaches $332.97 \pm 26.8$.
Further parameter analysis and hardware specification are given in Appendix.

\begin{figure}[t]
\centering
\includegraphics[scale=0.9]{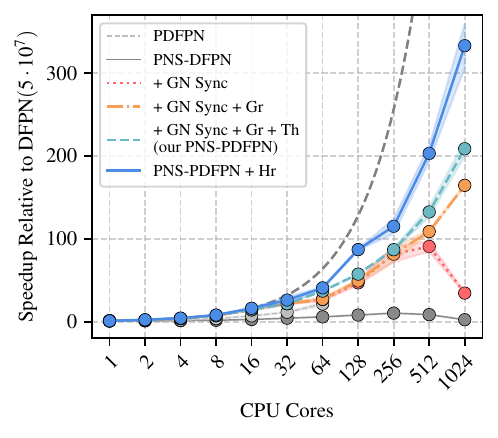}
\caption{
Impact of GN synchronization (GN Sync), grouping (Gr), second-level parallelization (Th), and the child-ordering heuristic (Hr) on the final speedup of PNS-PDFPN relative to DFPN($5 \cdot 10^7$), which finished in $293\pm41.7$ minutes. Measured on the 47-spot position.
}
\label{fig_third_experiment}
\end{figure}

\subsection{Solving Sprouts Positions}

With SPOTS, we determined 32 new outcomes of $n$-spot positions using the sequential version and 10 more with its parallel counterpart, nearly doubling the number of previously known results; see Appendix.
We verified the new proofs using the verification software by \citet{lemoine2015nimber}.
All computed outcomes agree with the Sprouts Conjecture, which thus remains open.
Compared to the previous state-of-the-art solver GLOP, our parallel version of SPOTS achieves a speedup of four orders of magnitude on 1024 cores.
We base this estimate on the combined improvements of the sequential SPOTS over GLOP and the parallel SPOTS over the sequential version.
To illustrate the speedup, SPOTS solves the $39$-spot position, one of the most complex solved by GLOP, in just 7.7 minutes using 1024 cores.
The outcomes determined by parallel SPOTS are much more complex than those previously computed.
The most difficult position was solved in 24 days on 512 cores, and its proof size (the number of stored Grundy numbers) exceeds prior ones by a factor of 1,000.

\begin{table}[t]
{\fontsize{9}{10}\selectfont
    \centering
    \begin{tabular}{rrrrcc}
        \toprule
        & Cores & \multicolumn{2}{c}{Parameters} & \multicolumn{2}{c}{Speedup} \\
        \cmidrule(r){3-4}
        \cmidrule(lr){5-6}
        & & \multicolumn{1}{c}{It} & \multicolumn{1}{c}{Th} & PNS-PDFPN & + Heuristic \\
        \cmidrule(r){1-6}
        \multirow{11}{*}{\rotatebox{90}{\color{gray}{Best Values for No-Heuristic}}}
        &    1 &  1k &  1 & \multicolumn{1}{r}{ \textbf{    1.08}\sem{0.10}   } & \multicolumn{1}{r}{            0.99\sem{0.00}    } \\
        &    2 &  1k &  1 & \multicolumn{1}{r}{             1.87\sem{0.02}  }   & \multicolumn{1}{r}{ \textbf{   2.36}\sem{0.28}   } \\
        &    4 &  1k &  1 & \multicolumn{1}{r}{             3.91\sem{0.31}  }   & \multicolumn{1}{r}{ \textbf{   4.34}\sem{0.32}   } \\
        &    8 &  1k &  1 & \multicolumn{1}{r}{ \textbf{    7.41}\sem{0.21}   } & \multicolumn{1}{r}{            7.10\sem{0.53}  } \\
        &   16 & 10k &  1 & \multicolumn{1}{r}{ \textbf{   14.37}\sem{0.86}   } & \multicolumn{1}{r}{            10.15\sem{4.48} }  \\
        &   32 & 10k &  1 & \multicolumn{1}{r}{            21.75\sem{0.68}  }   & \multicolumn{1}{r}{ \textbf{   23.26}\sem{0.05}  } \\
        &   64 &  1k &  4 & \multicolumn{1}{r}{            36.97\sem{0.29}  }   & \multicolumn{1}{r}{ \textbf{   37.93}\sem{0.21}      } \\
        &  128 &  1k &  8 & \multicolumn{1}{r}{ \textbf{   57.43}\sem{0.05}   } & \multicolumn{1}{r}{            45.76\sem{2.32}      } \\
        &  256 &  1k &  8 & \multicolumn{1}{r}{            86.87\sem{1.42}  }   & \multicolumn{1}{r}{ \textbf{   92.35}\sem{1.57}  } \\
        &  512 &  1k &  8 & \multicolumn{1}{r}{ \textbf{  132.44}\sem{6.11}   } & \multicolumn{1}{r}{           123.72\sem{3.97}      } \\
        & 1024 &  1k & 16 & \multicolumn{1}{r}{ \textbf{  208.67}\sem{9.17}   } & \multicolumn{1}{r}{           169.24\sem{2.86}      } \\        
        \cmidrule(r){1-6}
        \multirow{11}{*}{\rotatebox{90}{\color{gray}{Best Values for Heuristic}}}
        &    1 &   1k &   1 &  \multicolumn{1}{r}{  \textbf{  1.08}\sem{0.10}   } & \multicolumn{1}{r}{              0.99\sem{0.00}    } \\
        &    2 &   1k &   1 &  \multicolumn{1}{r}{            1.87\sem{0.02}    } & \multicolumn{1}{r}{ \textbf{     2.36}\sem{0.28}   } \\
        &    4 &   1k &   1 &  \multicolumn{1}{r}{            3.91\sem{0.31}    } & \multicolumn{1}{r}{ \textbf{     4.34}\sem{0.32}   } \\
        &    8 &  10k &   1 &  \multicolumn{1}{r}{            7.36\sem{0.41}    } & \multicolumn{1}{r}{ \textbf{     8.06}\sem{0.09}   } \\
        &   16 & 100k &   1 &  \multicolumn{1}{r}{           11.90\sem{0.94}    } & \multicolumn{1}{r}{ \textbf{    16.22}\sem{0.53}   } \\
        &   32 & 100k &   4 &  \multicolumn{1}{r}{           18.50\sem{0.76}    } & \multicolumn{1}{r}{ \textbf{    26.13}\sem{3.87}   } \\
        &   64 & 100k &   8 &  \multicolumn{1}{r}{           27.95\sem{2.19}    } & \multicolumn{1}{r}{ \textbf{    40.86}\sem{1.13}   } \\
        &  128 & 100k &   4 &  \multicolumn{1}{r}{           29.44\sem{0.53}    } & \multicolumn{1}{r}{ \textbf{    87.02}\sem{0.81}   } \\
        &  256 & 100k &   8 &  \multicolumn{1}{r}{           54.81\sem{3.71}    } & \multicolumn{1}{r}{ \textbf{   115.11}\sem{8.88}   } \\
        &  512 & 100k &  16 &  \multicolumn{1}{r}{           73.21\sem{6.36}    } & \multicolumn{1}{r}{ \textbf{   203.05}\sem{9.82}   } \\
        & 1024 & 100k &  32 &  \multicolumn{1}{r}{           97.60\sem{9.24}    } & \multicolumn{1}{r}{ \textbf{  332.97}\sem{26.8}   } \\       
        \bottomrule
    \end{tabular}
    \caption{Best-performing values of \texttt{iterations} (It) and \texttt{threads} (Th) for PNS-PDFPN without (top) and with (bottom) the heuristic.
    \texttt{Grouping} set to its maximum and \texttt{updates} to 1,000 are optimal across all experiments.}
    \label{tab:heuristicEvaluation}
}
\end{table}

\section{Conclusion}

While our experiments focus on the impartial game Sprouts, PNS-PDFPN is domain-independent and applicable to any combinatorial game.
For non-impartial games, the solver naturally reduces to exploring standard NAND trees, where computed Grundy numbers correspond to loss outcomes.
Position decomposition is the only impartial-specific component, used as an optional acceleration in all baselines, and therefore does not affect overall scalability.
Scalability arises from multi-level parallelization and intermediate-result sharing, both of which are general.
Sprouts offers one advantage for scaling: its relatively slow expansion rate allows frequent synchronization.
However, this does not favor our algorithm, which strongly outperforms state-of-the-art methods in the same setting.
In faster-expanding domains, we suggest using costlier search-guiding heuristics in exchange for improved scalability.
Lastly, since even minimal domain knowledge for job assignment notably boosts scaling, more sophisticated use of domain knowledge, such as RL-based policies, may yield further speedups.

\appendix

\section{Acknowledgments}
All authors were supported by the grant no.\ 25-18031S of the Czech Science Foundation (GA\v{C}R).
T. \v{C}\'{i}\v{z}ek was supported by the Charles University Grant Agency (GAUK) project no.\ 326525.
M. Schmid was also supported by the Charles University project UNCE 24/SCI/008.
The authors thank EquiLibre Technologies, Inc.\ for providing computational resources.
Additional computational resources were provided by the e-INFRA CZ project (ID:90254), supported by the Ministry of Education, Youth and Sports of the Czech Republic.
Special thanks to Neil Burch for valuable comments and insights.

\bibliography{aaai2026}

\clearpage

\section{Sprouts Game}

Sprouts is a well-known combinatorial pencil-and-paper game introduced by Conway and Paterson, serving as the domain for our experiments. 
Here, we provide an overview of the game, including its rules, properties, and historical background, and describe our efforts to solve its positions.

\subsection{Rules of the Game}

\emph{Sprouts} start with $n$ initial spots arbitrarily placed on a sheet of paper.
The players then alternate in connecting the spots by curves according to the following simple rules:

\begin{itemize}
    \item Curves either connect two different spots or form loops at a single spot.
    \item No curve can cross or touch itself or any other curve except at the endpoints.
    \item Each spot can be incident to at most three curves; a loop is counted twice.
    \item After a curve is drawn, the same player also places a new spot along it.
    \item The first player who cannot make a move loses the game.
\end{itemize}

Figure~\ref{fig_example_game} contains an example of a Sprouts game with two initial spots.
Here, the first player loses the game since the last two spots in the final position, which are both incident to less than three curves, cannot be connected without crossing another curve.

\begin{figure}[!htb]
\centering
\begin{tabular}{>{\centering\arraybackslash}m{0.28\linewidth}>{\centering\arraybackslash}m{0.28\linewidth}>{\centering\arraybackslash}m{0.28\linewidth}}
    \includegraphics[scale=1]{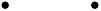} &    
    \includegraphics[scale=1]{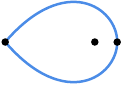} &
    \includegraphics[scale=1]{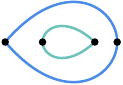} \\[6pt]
(1) & (2) & (3) \\[15pt]
\multicolumn{3}{c}{
    \begin{tabular}{>{\centering\arraybackslash}m{0.28\linewidth}>{\centering\arraybackslash}m{0.28\linewidth}}
        \includegraphics[scale=1]{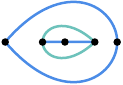} &
        \includegraphics[scale=1]{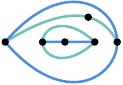} \\[6pt]
        (4) & (5)
    \end{tabular}
 }\\[6pt]
\end{tabular}
\caption{An example of a Sprouts game.}
\label{fig_example_game}
\end{figure}

Sprouts is an impartial game played under the normal play convention.
In particular, it ends after a finite number of moves, and there is no draw.
It follows that there is a winning strategy for either the first or the second player in every position.

\subsection{History and Related Work}

Sprouts was designed by British mathematicians John Horton Conway and Michael Paterson in 1967 with the intention to create a simple-to-play but difficult-to-analyze game~\cite{berlekamp2003winning, roberts2015genius}.

Since the early beginnings of Sprouts, researchers have tried to determine the outcomes of $n$-spot positions for as many values of $n$ as possible.
Conway himself found the outcomes for $n \leq 3$.
Mollison later proved the outcomes for $n$ equal to 4 and 5, as well as for $n = 6$, whose proof by hand took 49 pages \cite{gardner1967mathematical}.

\citet{applegate1991computer} created the first computer solver for Sprouts, which proved the outcomes for $n \leq 11$, and led the authors to formulate the Sprouts Conjecture.
Their solver used Alpha-Beta Pruning with the string representation of positions on which all subsequent solvers, including ours, were built.
In 2006, Josh Purinton created a computer program, called \emph{AuntBeast}, that played casual games with players from the \emph{World Game Of Sprouts Association} and was able to play a perfect game up to 14 spots; thus, newly solving the game for $n=12, 13,14$.

The previous state-of-the-art solver GLOP, created by \citet{lemoine2015nimber}, solved all $n$-spot positions for $n \leq 32$ in 2007.
In 2011, they extended the results for $n \leq 44$ plus three more values to $n = 53$; see Table~\ref{tab_new_records}.
After hearing about these achievements made by Lemoine and Viennot, Conway, one of the authors of Sprouts, reacted with the following disbelieving words~\cite{roberts2015genius}:

\begin{displayquote}
\emph{``I doubt that very much. They are basically saying they have done the impossible. If someone says they’ve invented a machine that can write a play worthy of Shakespeare, would you believe them? It’s just too complicated. If someone said they’d been having some success teaching pigs to fly... Though if they were doing that over in the field behind the Institute [for Advanced Study in Princeton], I would like to take a look.''}
\end{displayquote}

The critical improvement by Lemoine and Viennot compared to the previous solver by \citet{applegate1991computer} stands on the utilization of Grundy numbers to analyze independent subpositions separately and on human interventions to navigate the search outside of unpromising parts of a game tree.
Their algorithm originally relied on Alpha-Beta Pruning with a heuristic for ordering child-node exploration.
Later, Lemoine and Viennot replaced Alpha-Beta with a basic variant of Proof-Number Search, enabling them to achieve several new results.

\subsection{Solving Sprouts Positions}

\begin{table*}[t]
    \centering
    \begin{tabular}{|cccc|cccc|cccc|cccc|cccc|}

    \Xhline{1pt}
    \raisebox{1.4pt}[2.5ex][0pt]{\textbf{n}} & \raisebox{1.4pt}[2.5ex][0pt]{\textbf{out}} & \raisebox{1.4pt}[2.5ex][0pt]{\textbf{by}} & \raisebox{1.4pt}[2.5ex][0pt]{\textbf{size}} &
\raisebox{1.4pt}[2.5ex][0pt]{\textbf{n}} & \raisebox{1.4pt}[2.5ex][0pt]{\textbf{out}} & \raisebox{1.4pt}[2.5ex][0pt]{\textbf{by}} &  \raisebox{1.4pt}[2.5ex][0pt]{\textbf{size}} &
\raisebox{1.4pt}[2.5ex][0pt]{\textbf{n}} & \raisebox{1.4pt}[2.5ex][0pt]{\textbf{out}} & \raisebox{1.4pt}[2.5ex][0pt]{\textbf{by}} &  \raisebox{1.4pt}[2.5ex][0pt]{\textbf{size}} &
\raisebox{1.4pt}[2.5ex][0pt]{\textbf{n}} & \raisebox{1.4pt}[2.5ex][0pt]{\textbf{out}} & \raisebox{1.4pt}[2.5ex][0pt]{\textbf{by}} &  \raisebox{1.4pt}[2.5ex][0pt]{\textbf{size}} &
\raisebox{1.4pt}[2.5ex][0pt]{\textbf{n}} & \raisebox{1.4pt}[2.5ex][0pt]{\textbf{out}} & \raisebox{1.4pt}[2.5ex][0pt]{\textbf{by}} &  \raisebox{1.4pt}[2.5ex][0pt]{\textbf{size}} \\
    \Xhline{1pt}
     \conw{\rule{0pt}{2.2ex}1}{L}{2e0}    &    \glop{22}{W}{6e3}    &       \glopNew{43}{L}{1e6}         & \spots{64}{W}{3e6}        & \spots{85}{L}{2e7}              \\
     \conw{2}{L}{4e0}    &    \glop{23}{W}{4e3}    &       \glopNew{44}{L}{1e6}         & \spots{65}{W}{1e7}        & \spots{86}{L}{2e7}          \\
     \conw{3}{W}{7e0}    &    \glop{24}{L}{5e4}    &         \spots{45}{W}{3e6}         & \none{66}                 & \noneBorder{87}              \\
     \moll{4}{W}{2e1}    &    \glop{25}{L}{2e4}    &       \glopMid{46}{W}{2e5}         & \spots{67}{L}{1e7}        & \spots{88}{W}{2e7}          \\
     \moll{5}{W}{3e1}    &    \glop{26}{L}{4e4}    &          \glop{47}{W}{2e5}         & \spots{68}{L}{1e7}        & \parallelspots{89}{W}{1e9}  \\
     \moll{6}{L}{9e1}    &    \glop{27}{W}{3e5}    &         \spots{48}{L}{3e6}         & \none{69}                 & \noneBorder{90}              \\
      \ajs{7}{L}{2e2}    &    \glop{28}{W}{1e4}    &         \spots{49}{L}{3e6}         & \spots{70}{W}{3e6}        & \spots{91}{L}{2e7}          \\
      \ajs{8}{L}{3e2}    &    \glop{29}{W}{1e4}    &         \spots{50}{L}{3e6}         & \spots{71}{W}{1e7}        & \spots{92}{L}{2e7}          \\
      \ajs{9}{W}{1e1}    &    \glop{30}{L}{2e5}    &         \spots{51}{W}{1e7}         & \none{72}                 & \noneBorder{93}              \\
     \ajs{10}{W}{3e2}    &    \glop{31}{L}{5e4}    &         \spots{52}{W}{3e6}         & \spots{73}{L}{1e7}        & \parallelspots{94}{W}{1e9}  \\
     \ajs{11}{W}{2e2}    &    \glop{32}{L}{7e4}    &       \glopMid{53}{W}{8e5}         & \spots{74}{L}{1e7}        & \parallelspots{95}{W}{1e9}  \\
    \purr{12}{L}{1e3}    & \glopMid{33}{W}{1e6}    &         \spots{54}{L}{1e7}         & \none{75}                 & \noneBorder{96}              \\
    \purr{13}{L}{1e3}    &    \glop{34}{W}{3e4}    &         \spots{55}{L}{3e6}         & \spots{76}{W}{1e7}        & \parallelspots{97}{L}{1e9}  \\
    \purr{14}{L}{1e4}    &    \glop{35}{W}{3e4}    &         \spots{56}{L}{3e6}         & \spots{77}{W}{1e7}        & \parallelspots{98}{L}{1e9}  \\
    \glop{15}{W}{9e4}    & \glopMid{36}{L}{1e6}    & \parallelspots{57}{W}{1e9}         & \none{78}                 & \noneBorder{99}              \\
    \glop{16}{W}{1e3}    & \glopMid{37}{L}{1e5}    &         \spots{58}{W}{3e6}         & \spots{79}{L}{1e7}        & \parallelspots{100}{W}{1e9} \\
    \glop{17}{W}{6e2}    & \glopMid{38}{L}{1e5}    &         \spots{59}{W}{7e5}         & \spots{80}{L}{1e7}        & \noneBorder{101}             \\
    \glop{18}{L}{7e3}    & \glopNew{39}{W}{1e6}    & \parallelspots{60}{L}{1e9}         & \none{81}                 & \noneBorder{102}             \\
    \glop{19}{L}{9e3}    &    \glop{40}{W}{1e5}    &         \spots{61}{L}{3e6}         & \spots{82}{W}{1e7}        & \parallelspots{103}{L}{1e9} \\
    \glop{20}{L}{1e4}    &    \glop{41}{W}{2e5}    &         \spots{62}{L}{3e6}         & \spots{83}{W}{1e7}        & \parallelspots{104}{L}{1e9} \\
    \glop{21}{W}{8e4}    & \glopNew{42}{L}{1e6}    &   \none{63}                        & \none{84}                 & \noneBorder{105}\\
     \hline
    \end{tabular}
    \caption{Solved $n$-spot Sprouts positions. Each entry in the $n$th row indicates the solver (C = Conway, M = Mollison, A = \citet{applegate1991computer}, P = AuntBeast (Purinton 2006), G = GLOP~\cite{lemoine2015nimber}, S = SPOTS), the outcome (W = win, L = loss), and size of the proof measured by the number of computed Grundy numbers. Lighter and darker G-entries indicate positions solved by GLOP in 2007 and in 2010--2011, respectively. 
    Lighter and darker S-entries indicate positions solved by the sequential and parallel versions of our SPOTS solver, respectively.}
    \label{tab_new_records}
\end{table*}

We derived 32 new outcomes with sequential SPOTS and 10 more new outcomes of much more complex positions with parallel SPOTS.
Thus, we extended the known outcomes of $n$-spot positions by 42 new values in total, almost doubling the number to the final 89 known values.
Since the complexity does not grow monotonically, as positions with $n \equiv 3 \pmod{6}$ are particularly challenging, we established outcomes for all $n \leq 62$ and the $27$ least complex cases with $63 \leq n \leq 104$.
For a detailed list of all the currently known outcomes of $n$-spot positions, see Table~\ref{tab_new_records}.

All the newly computed outcomes agree with the Sprouts conjecture, which thus remains open.
Additionally, we confirm that the computed Grundy numbers of $n$-spot positions follow the extended Sprouts conjecture of \citet{lemoine2015nimber}, which states that the Grundy number of the $n$-spot position with $n = 0, 1, \text{or } 2 \pmod{6}$ equals to 0, and to 1 otherwise.

A server workstation equipped with an AMD EPYC 7302 processor (3.00 GHz, 32 cores) and 256 GB of RAM, running Debian GNU/Linux 12, was used for long-running sequential computations.
Massively parallel computations were performed on the \emph{Google Compute Engine} platform, part of the \emph{Google Cloud Platform}, which enables the creation of a large number of interconnected virtual machines (VMs).
We set up a cluster composed of a single standard VM for the master and multiple spot VMs for workers, with automatic restoration after each preemption.
All VMs were located in the \emph{us-central1} region and ran Ubuntu 22.04 LTS.
The standard VM used the \emph{c2d-highmem-32} instance type, providing 16 cores, 256 GB of RAM, and 256 GB of balanced persistent disk.
For workers, we allocated 16 spot VMs of the \emph{c2-standard-60} type, each with 30 cores, 240 GB of RAM, and 128 GB of balanced persistent disk.

After a month of computations on this cluster with 480 cores allocated to workers, we derived proofs of Sprouts positions that are approximately 1,000 times larger than the most complicated ones derived by the current state-of-the-art solver GLOP \cite{lemoine2015nimber,lemoine2011phdthesis}.
We computed around $10^9$ Grundy numbers in total, where one Grundy number corresponds to approximately 100 node expansions.

Notably, the results were achieved without fully utilizing the potential of our solver, as parameters were not always optimally tuned, and some workers were intermittently terminated during computation due to spot machine preemption.
The parallel SPOTS computations were executed on 480 CPU cores, primarily using the following parameter settings:
\texttt{workers} = 48, \texttt{iterations} = 10,000, \texttt{updates} = 1,000, \texttt{grouping} = 3, and \texttt{threads} = 10.

\paragraph{Code and Data.}
The project’s GitHub repository contains the source code for our parallel Sprouts solver, SPOTS. It also provides links to the Grundy number databases, which serve as certificates for newly solved positions.

All computation-demanding parts of the solver, such as the representation of a state and search algorithms, are implemented in C{+}{+}20~for performance reasons.
The first-level work distribution is implemented in \emph{Python}~$3.10$ using \emph{Ray}~$2.47.0$, a framework for distributed systems.
The master and the workers are coordinated using Ray, while they utilize the core functions implemented in C{+}{+}, exposed to Python using \emph{pybind11}~$2.10$.
The whole project is built using \emph{CMake}~$3.22$ with the compiler \emph{GCC}~$11.4$.

The SPOTS solver can compute the outcome of any given Sprouts position. Positions are provided as input using the string representation introduced by \citet{applegate1991computer} and later used by \citet{lemoine2015nimber}. 
In particular, the $n$-spot position is encoded as \texttt{0*}$n$.

\subsection{Estimating the Game Tree Complexity}

We describe how the estimates of the game tree complexity in Figure~\ref{fig_tree_complexity} were obtained.
For Checkers, Chess, Shogi, and Go, we followed a standard approach based on known estimates of the average branching factor $b$ and average depth $d$ of the game tree.
The complexity was then estimated as $b^d$.

To estimate the mean and variance of game tree complexity for Sprouts, we sampled 1,000 random games by selecting each child uniformly at random at each node along a path $\mathcal{P}$ from the root to a terminal node. 
Let $b_i$ denote the number of children of the $i$th node along $\mathcal{P}$.
We computed $\prod_{i=1}^k b_i$ for each path, where $k$ is the number of internal nodes in the path. 
The final estimate was then obtained by averaging these values across all the samples, with the range between the 1st and 99th percentiles as the measure of variance.

To estimate the complexity of the Sprouts tree with Grundy numbers, we used a similar method based on randomly sampled games, but with an important modification: decomposable nodes must be treated separately, as their total complexity is the sum of the complexities of subtrees of all their children. 
The estimation is therefore recursive, and instead of the path $\mathcal{P}$ in the tree, we obtain a randomly sampled subtree $\mathcal{T}$. 
As in our description of DFPN with GN, we first simplify the Sprouts tree by replacing Grundy nodes with atomic nodes, resulting in a tree that includes only atomic and decomposable nodes. 
The subtree $\mathcal{T}$ is then obtained by starting at the root, selecting a single child at every encountered atomic node uniformly at random, and by selecting all children of each encountered decomposable node.
After sampling $\mathcal{T}$, we initially set the complexity $C(t)$ of each terminal $t$ in $\mathcal{T}$ to be 1.
To compute the complexity $C(v)$ of the Sprouts subtree of an internal node $v$ of $\mathcal{T}$, we distinguish two cases.
If $v$ is atomic, then we let $C(v) = b(v) \cdot C(c)$, where $c$ is the child of $v$ in $\mathcal{T}$ and $b(v)$ is the number of children of $v$ in the Sprouts tree.
For decomposable $v$, we let $C(v) = \sum_c C(c)$, where the sum is taken over all the children $c$ of $v$ in the Sprouts tree.
The complexity estimate of the Sprouts tree with Grundy numbers is then the average of the values $C(r)$ of the root $r$ taken over 1,000 samples of $\mathcal{T}$, with the range between the 1st and 99th percentiles reported as the measure of variance.
During the sampling, we need to estimate the sizes of the Grundy numbers of the nodes.
However, these values are unknown.
Using the average values of the Grundy numbers in the computed proof trees of $n$-spot positions, we estimate the number of relevant children of Grundy nodes to determine their value.

\section{Proof of Theorem~\ref{thm-dfpn}}

Here, we prove Theorem~\ref{thm-dfpn} by showing that our choice of the thresholds in DFPN with GN guarantees that, for every node $v$ of $\mathcal{P}$, the MPN is in the subtree of $v$ if and only if the following three inequalities are satisfied: 
\begin{align}
\label{eq-thresholds}
    \begin{split}
    pt(v) &> pn(v)\\ 
    dt(v) &> dn(v),\\
    mt(v) &> \min \{pn(v) + ps(v), dn(v) + ds(v)\}.
    \end{split}
\end{align}
We also provide some intuition for how the thresholds are selected.

We first recall how the proof and disproof numbers are defined.
The proof and disproof numbers of an internal atomic node $v$ are given by
\begin{equation}
\label{eq-pndnAtomic}
pn(v) = \min_{c}{dn(c)} \;\;\;\text{ and }\;\;\; dn(v) = \sum_{c}{pn(c)},
\end{equation}
where the minimum and the sum are taken over the children $c$ of $v$.
If $u$ is a Grundy node, then it corresponds to an atomic position~$Q$ and we use $v_i$ to denote the child of $u$ representing $Q+\ast i$.
If $v_j$ is the last child of $u$ generated so far, then we set 
\begin{equation}
\label{eq-pnGrundy}
pn(u) = dn(u) = \min{\{pn(v_j), dn(v_j)\}}.
\end{equation}
Let $v$ be a decomposable internal node.
Then, $v$ corresponds to $P_1+\dots+P_k$, where $v'_k$ is an atomic node corresponding to $P_k + \ast n\oplus gn(P_1)\oplus\cdots\oplus gn(P_{k-1})$ and $u_1,\dots,u_k$ denote the nodes where each $u_i$ is a Grundy node corresponding to~$P_i$.
The proof and disproof numbers of $v$, are then given by
\begin{equation}
\label{eq-pndnDecomposable}
pn(v) = dn(v) = \sum_{i = 1}^k pn(u_i)
\end{equation}
if $v'_k$ has not yet been generated, and by
\begin{equation}
\label{eq-pndnDecomposableGenerated}
pn(v) = pn(v'_k) \;\;\;\text{ and }\;\;\; dn(v) = dn(v'_k)
\end{equation}
otherwise.

Now, let $v$ be a node from the path $\mathcal{P}$, and let $w$ be the next node to be selected. 
If $v$ is atomic, then we let $w'$ be its child with the second-lowest disproof number.
For any node~$v'$, we use $pn_0(v')$ and $dn_0(v')$ to denote the values of $pn(v')$ and $dn(v')$, respectively, before selecting $w$ in $v$.
Since all updates of proof and disproof numbers go along the path $\mathcal{P}$ and $w$ is the only child of $v$ on $\mathcal{P}$, we always have $dn_0(c) = dn(c)$ for every child $c \neq w$ of $v$.
Note that, by combining this fact with~\eqref{eq-pndnAtomic}, we get that an internal atomic $v$ satisfies
\begin{equation}
\label{eq-zeroDnAtomic}
dn_0(v) = dn(v) + pn_0(w) - pn(w).
\end{equation}

Similarly, we obtain a relationship between $pn_0(v)$ and $pn(v)$ for decomposable $v$ using~\eqref{eq-pnGrundy}.
In particular, if $v$ is decomposable and $v'_k$ has not yet been generated, then
\begin{align}
\label{eq-zeropnDecomposable}
\begin{split}
pn(v) &= dn(v) = \sum_c \min \{pn(c),dn(c)\} \\
&= pn_0(v) - \min \{pn_0(w),dn_0(w)\} \\
&+ \min \{pn(w),dn(w)\}.
\end{split}
\end{align}

We also recall the definitions of the thresholds of nodes in $\mathcal{P}$.
For atomic~$v$, the thresholds $pt(w)$, $dt(w)$, and $mt(w)$ with parameters $ps(w)$ and $ds(w)$ are defined as
\begin{align}
\label{eq-atomicThresholds}
    \begin{split}
            pt(w) &= dt(v) - dn_0(v) + pn_0(w),\\
            dt(w) &= \min \{pt(v), dn_0(w') + 1\},\\
            mt(w) &= mt(v),
    \end{split}
\end{align}    
where the parameters $ps(w)$ and $ds(w)$ are set to $ps(w) = ds(v) + dn_0(v) - pn_0(w)$ and $ds(w) = ps(v)$.
If the node $v$ is decomposable, then we set 
\begin{equation}
\label{eq-decomposableThresholdsGenerated}
pt(w)=pt(v), \;\;\; dt(w)=dt(v), \;\;\; mt(w)=mt(v)    
\end{equation}
if $v'_k$ has been generated.
If $v'_k$ has not been generated yet, we define
\begin{align}
\label{eq-decomposableThresholds}
    \begin{split}
            pt(w) &= dt(w) = \infty,\\
            mt(w) &= t(v) - pn_0(v) + \min \{pn_0(w), dn_0(w) \},
    \end{split}
\end{align}
where we let $ps(w) = ds(w) = 0$ and $t(v) = \min \{ pt(v),\allowbreak dt(v), mt(v) - \min \{ ps(v),ds(v) \} \}$.
For the root~$r$, we set $pt(r)=dt(r)=mt(r) =\infty$ and $ps(r)=ds(r)=0$.

To gain some insight about the choice~\eqref{eq-thresholdsW} of thresholds, note that the definition~\eqref{eq-atomicThresholds} of $pt(w)$ and $dt(w)$ for atomic $v$ is exactly the same as in DFPN without GN.
The first value in $dt(w)$ and the value of $pt(w)$ indicate that DFPN with GN must backtrack when the subtree of $v$ no longer contains the MPN.
The second value in $dt(w)$ indicates that when the MPN switches to the child $w'$.
The new threshold $mt(w)$ is needed because it follows from~\eqref{eq-zeropnDecomposable} that to have $pt(v) > pn(v)$ and $dt(v) > dn(v)$ for decomposable $v$ we need an upper bound on $\min\{pn(w),dn(w)\}$, which is eventually ensured by $mt(w)$ in combination with $ps(w)$ and $ds(w)$, which serve for shifting the values in the atomic node $v$.

\begin{proof}[Proof of Theorem~\ref{thm-dfpn}]
We proceed by induction on the depth $d$ of a node $v$ in $\mathcal{P}$ and prove that MPN is in the subtree of $v$ if and only if the inequalities~\eqref{eq-thresholds} are satisfied.
Since $v$ is in $\mathcal{P}$, it is an internal node.
We note that the MPN is in the subtree of $v$ if and only if all nodes $v'$ on the subpath of $\mathcal{P}$ from the root to $v$ satisfy the following condition: if $v'$ has atomic parent $p$, then $v'$ has the smallest disproof number among all children of $p$.

For the induction base, we have $d=0$, in which case our node is the root $r$.
If $r$ has not yet been solved, then the MPN is always in the subtree of $r$, and the inequalities~\eqref{eq-thresholds} are satisfied as $pt(r) = \infty > pn(r)$, $dt(r) = \infty > dn(r)$, and $\min \{pn(r) + 0, dn(r) + 0\} < \infty = mt(r)$.
If $r$ is solved, then obviously the MPN is not in the subtree of $r$, and the threshold inequalities are not satisfied, since $\infty < \infty$ is interpreted as false.

For the induction step, we assume $d \geq 1$ and prove that the MPN is in the subtree of the next selected node $w$ if and only if
\begin{align}
\label{eq-thresholdsW}
    \begin{split}
    pt(w) &> pn(w)\\ 
    dt(w) &> dn(w),\\
    mt(w) &> \min \{pn(w) + ps(w), dn(w) + ds(w)\}.
    \end{split}
\end{align}
We assume that the claim is true for the parent $v$ of $w$.
Note that the parent $v$ exists as $d \geq 1$.

Assume first that the MPN is in the subtree of $w$.
We show that the three inequalities~\eqref{eq-thresholdsW} are satisfied.
Since $w$ is a child of $v$, the subtree of $w$ is contained in the subtree of $v$, and therefore the MPN is in the subtree of $v$.
It then follows from the induction hypothesis that the inequalities~\eqref{eq-thresholds} are satisfied.

We now distinguish two cases.
First, we assume that $v$ is atomic.
Since the MPN is in the subtree of $w$, the node $w$ has the smallest disproof number among all children of $v$.
In particular, $dn(w) \leq dn(w') = dn_0(w')$, because the MPN is in the subtree of $w$.

To prove the first inequality in~\eqref{eq-thresholdsW}, $pt(w) = dt(v) - dn_0(v) + pn_0(w)$ by~\eqref{eq-atomicThresholds}.
Substituting for $dn_0(v)$ using~\eqref{eq-zeroDnAtomic}, we rewrite this equality as $pt(w) = dt(v) - dn(v) + pn(w)$.
Since $dt(v) > dn(v)$ by~\eqref{eq-thresholds}, we have $pt(w)>pn(w)$, which satisfies the first inequality of~\eqref{eq-thresholdsW}.

Second, we use $dt(w) = \min \{pt(v), dn_0(w') + 1\}$, which holds by~\eqref{eq-atomicThresholds}.
Using $dn(w) \leq dn(w') = dn_0(w')$ and $pt(v) > pn(v)$, which follows from~\eqref{eq-thresholds}, we get $dt(w) \geq \min \{pn(v)+1, dn(w) + 1\}$. 
Since $v$ is atomic, we have $pn(v) \geq dn(w)$ by~\eqref{eq-pndnAtomic}.
Altogether, this gives the second inequality $dt(w) > dn(w)$ from~\eqref{eq-thresholdsW}.

To prove the last inequality of~\eqref{eq-thresholdsW}, we first show that $mt(w)=mt(v)$ and that $\min \{pn(v) + ps(v), dn(v) + ds(v)\}=\min \{pn(w) + ps(w), dn(w) + ds(w)\}$.
First, we have $mt(w)=mt(v)$ by~\eqref{eq-atomicThresholds}.
Also, $ps(w) = ds(v) + dn_0(v)-pn_0(w)$, which becomes $ps(w) = ds(v) + dn(v)-pn(w)$ after substituting for $dn_0(v)$ using~\eqref{eq-zeroDnAtomic}.
So, we get $pn(w) + ps(w) = dn(v) + ds(v)$.
Moreover, we have $ds(w) = ps(v)$ and so $dn(w) + ds(w) = dn(w) + ps(v)$.
Using the definition~\eqref{eq-pndnAtomic} of $pn(v)$, we get $pn(v)=dn(w)$, since $w$ has the smallest disproof numbers among all children of the atomic node $v$, and so we get the desired equality $dn(w) + ds(w) = pn(v) + ps(v)$.
Thus, we indeed have $mt(w)=mt(v)$ and $\min \{pn(v) + ps(w), dn(w) + ds(w)\}=\min \{pn(w) + ps(w), dn(w) + ds(w)\}$.
Then we are done, as the third inequality~\eqref{eq-thresholds} is equivalent to the third inequality in~\eqref{eq-thresholdsW}.

\medskip

Now, assume that $v$ is decomposable.
We distinguish two subcases, depending on whether the last child $v'_k$ of $w$ has been generated.

Assume first that $v'_k$ has been generated.
Then $w=v'_k$ by the choice of the selected node in the algorithm.
According to~\eqref{eq-pndnDecomposableGenerated}, we then have $pn(v)=pn(w)$ and $dn(v)=dn(w)$.
Moreover, $pt(v)=pt(w)$, $dt(v)=dt(w)$, and $mt(v)=mt(w)$ by~\eqref{eq-decomposableThresholdsGenerated}.
Since all values are the same for $v$ and $w$, all inequalities from~\eqref{eq-thresholdsW} follow immediately from~\eqref{eq-thresholds}.

For the other subcase, assume that $v'_k$ has not yet been generated.
Then, the first two inequalities in~\eqref{eq-thresholdsW} are trivially satisfied, as $pt(w)=dt(w)=\infty$ while $pn(w)$ and $dn(w)$ are finite, since the subtree of $w$ contains the MPN.
To prove the third inequality in~\eqref{eq-thresholdsW}, it suffices to show that $mt(w)>\min\{pn(w),dn(w)\}$, as $ps(w)=ds(w)\allowbreak =0$ for decomposable $v$.
We recall that 
\[mt(w) = t(v) - pn_0(v) + \min \{pn_0(w), dn_0(w) \},\] where 
\[t(v) = \min \{ pt(v), dt(v), mt(v) - \min \{ ps(v),ds(v) \} \}.\]

We now prove that $t(v) > pn(v)=dn(v)$.
From~\eqref{eq-pndnDecomposable}, we obtain $pn(v)=dn(v)$.
Using~\eqref{eq-thresholds}, we estimate the first two terms in $t(v)$ from below as $pt(v)>pn(v)$ and $dt(v)>dn(v)=pn(v)$.
To estimate the third term $mt(v) - \min \{ ps(v),ds(v) \}$ in $t(v)$ with $pn(v)$, we first expand $mt(v)$ as
\begin{align*}
\begin{split}
mt(v) &= \min\{pn(v)+ps(v),dn(v)+ds(v)\}\\
&= pn(v) + \min\{ps(v),ds(v)\}\,
\end{split}
\end{align*}
where we used $pn(v)=dn(v)$.
The term $\min\{ps(v),\allowbreak ds(v)\}$ then cancels out in $t(v)$ and we indeed obtain $t(v)>pn(v)$.

Using $t(v)>pn(v)$, we now have 
\[mt(w) > pn(v) - pn_0(v) + \min \{pn_0(w), dn_0(w)\}.\]
To finish the proof of $mt(w)>\min\{pn(w),dn(w)\}$, we express $pn(v)$ in terms of~\eqref{eq-zeropnDecomposable}, obtaining
\begin{align*}
\begin{split}
mt(w) &> pn_0(v) - \min \{pn_0(w),dn_0(w)\} \\
&+ \min \{pn(w),dn(w) - pn_0(v)\} \\
&+ \min \{pn_0(w), dn_0(w)\}\\
&= \min \{pn(w),dn(w) - pn_0(v)\},
\end{split}
\end{align*}
which finishes the proof of the last inequality in~\eqref{eq-thresholdsW}.

\bigskip

For the other implication, assume that the inequalities~\eqref{eq-thresholdsW} are satisfied.
We again distinguish two cases depending on the type of~$v$.

First, we assume that $v$ is atomic.
We aim to show that the MPN lies in the subtree of $w$.
To do this, it suffices to show that the inequalities~\eqref{eq-thresholds} hold, as this implies, by the induction hypothesis, that the MPN is in the subtree of $v$.
Since $w$ is the selected node, it is the child of $v$ with the smallest disproof number at the time of its selection.
That is, $dn_0(w) = \min_c dn_0(c)$ where the minimum is taken over all children $c$ of $v$.
We want to show that $w$ is also the child of $v$ with the smallest disproof number at all steps where $w$ is $\mathcal{P}$, as then the MPN indeed lies in the subtree of $w$.
This follows from the definition~\eqref{eq-atomicThresholds} of $dt(w)$, which implies $dt(w) \leq dn_0(w') + 1$, and thus $dn(w) < dt(w) \leq dn_0(w')+1=dn(w')+1$, where we use $dn_0(c)=dn(c)$ for every child $c \neq w$ o $v$ and the fact that $w'$ had the second-smallest disproof number among the children of $v$ at the time of selecting $w$ as the next node.

To prove the first inequality in~\eqref{eq-thresholds}, the definition~\eqref{eq-atomicThresholds} of $dt(w)$ gives $dt(w) = \min \{pt(v), dn_0(w') + 1\}$.
In particular, $pt(v) \geq dt(w)$.
We also have $dn(w) \geq pn(v)$, which follows from~\eqref{eq-pndnAtomic} and from the fact that $w$ is a child of $v$.
By the second inequality in~\eqref{eq-thresholdsW}, we then obtain $pt(v) \geq dt(w) > dn(w) \geq pn(v)$, which implies the first inequality in~\eqref{eq-thresholds}.

Second, we expand the definition~\eqref{eq-atomicThresholds} of $pt(w)$ to obtain $pt(w) = dt(v) - dn_0(v) + pn_0(w)$.
Using~\eqref{eq-zeroDnAtomic}, we rewrite this equality as $pt(w) = dt(v) - dn(v) + pn(w)$.
By the first inequality in~\eqref{eq-thresholdsW}, we then have $pt(w) = dt(v) - dn(v) + pn(w) > pn(w)$.
Equivalently, the second inequality $dt(v) > dn(v)$ in~\eqref{eq-thresholds} is true.

Finally, to prove the third inequality in~\eqref{eq-thresholds}, we first use the same arguments as in the proof of the first implication, and we derive $mt(w)=mt(v)$ and $\min \{pn(v) + ps(v), dn(v) + ds(v)\}=\min \{pn(w) + ps(w), dn(w) + ds(w)\}$.
The third inequality in~\eqref{eq-thresholdsW} is then equivalent to the third inequality in~\eqref{eq-thresholds}, and we are done.

\medskip

Now, assume that $v$ is decomposable.
We still aim to show that the MPN lies in the subtree of $w$.
Again, it suffices to show that the inequalities~\eqref{eq-thresholds} hold.
The induction hypothesis then implies that the MPN is in the subtree of $v$ and, since the selected node $w$ is the only generated child of $v$ that has not yet been solved, we see that the MPN is in the subtree of $w$.
We again distinguish two subcases, depending on whether $v'_k$ has been generated or not.

Assume first that $v'_k$ has been generated.
Then, as before, all values in~\eqref{eq-thresholds} and~\eqref{eq-thresholdsW} are the same, and thus all the inequalities in~\eqref{eq-thresholds} are satisfied.

For the other subcase, we assume that $v'_k$ has not yet been generated.
We prove all three inequalities in~\eqref{eq-thresholds} at the same time.
The third inequality in~\eqref{eq-thresholdsW} gives the following estimate:
\begin{align*}
\begin{split}
mt(w) &> \min \{pn(w) + ps(w), dn(w) + ds(w)\}\\
&=\min\{pn(w),dn(w)\},
\end{split}
\end{align*}
where we used $ps(w)=ds(w)=0$ for decomposable $w$.
We now expand the definition of $mt(w)$ as follows: 
\[mt(w) = t(v) - pn_0(v) + \min \{pn_0(w), dn_0(w) \},\] where 
\[t(v) = \min \{ pt(v), dt(v), mt(v) - \min \{ ps(v),ds(v) \} \}.\]
Using~\eqref{eq-zeropnDecomposable}, the expression of $mt(w)$ becomes
\begin{align*}
\begin{split}
mt(w) &= t(v) - pn(v) - \min\{pn_0(w),dn_0(w)\}\\
&+ \min\{pn(w),dn(w)\} + \min \{pn_0(w), dn_0(w)\}\\
&= t(v) - pn(v) + \min\{pn(w),dn(w)\}.
\end{split}
\end{align*}
Therefore, we know that
\[t(v) - pn(v) + \min\{pn(w),dn(w)\} > \min\{pn(w),dn(w).\]
This can be simplified as $t(v) > pn(v)$.
Thus, using the expression of $t(v)$, we derive 
\[\min \{ pt(v), dt(v), mt(v) - \min \{ ps(v),ds(v) \} \} > pn(v).\]
Since $pn(v)=dn(v)$ by~\eqref{eq-pndnDecomposable}, this implies the first two inequalities in~\eqref{eq-thresholds}.
For the third inequality, we obtained that 
\[mt(v) > pn(v) + \min \{ ps(v),ds(v)\}.\]
However, since $pn(v)=dn(v)$, we can rewrite this as
\[mt(v) > \min \{pn(v) + ps(v), dn(v) + ds(v)\},\]
obtaining the last inequality in~\eqref{eq-thresholds}.

This finishes the induction step and the proof of Theorem~\ref{thm-dfpn}.
\end{proof}

\begin{figure*}[!htb]
  \centering
  
  \begin{subfigure}[b]{0.48\textwidth}
    \centering
    \includegraphics{img/third_experiment.pdf}
    \label{fig:sub1}
  \end{subfigure}
  \hfill
  \begin{subfigure}[b]{0.42\textwidth}
    \centering
    \includegraphics{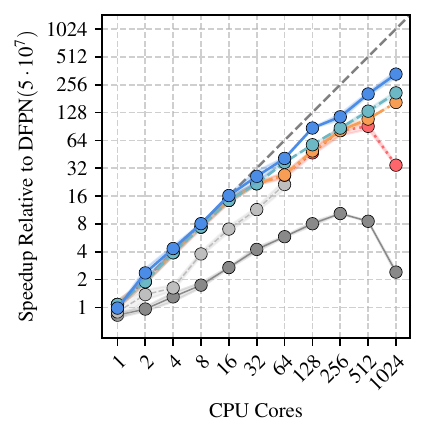}
    \label{fig:sub2}
  \end{subfigure}
  
  \caption{Impact of GN synchronization (GN Sync), grouping (Gr), second-level parallelization (Th), and the child-ordering heuristic (Hr) on the final speedup of PNS-PDFPN relative to DFPN($5 \cdot 10^7$).
Measured on the 47-spot position. 
The semi-log scale, with a logarithmic $x$-axis and a linear $y$-axis, is shown on the left; the log-log scale, with both axes logarithmic, is shown on the right.}
  \label{fig:main}
\end{figure*}

\section{Parameter Analysis}

We begin by recalling the final scaling efficiency of PNS-PDFPN with optimally tuned parameters, as shown in Figure~\ref{fig:main}, where we also include a logarithmic scale for a more detailed evaluation.
In this section, we analyze the optimal parameter settings of PNS-PDFPN by examining the \emph{search overhead}, defined as the ratio of node expansions performed by PNS-PDFPN to those performed by a sequential DFPN.
To capture the synchronization and communication overhead resulting from frequent interactions with the master, we also measure the \emph{worker utilization}, defined as the average proportion of time that workers are actively assigned to a job.

A cluster consisting of 25 nodes, each equipped with 2× AMD EPYC 9474F processors (3.60 GHz, 48 cores per processor), 1536 GB of RAM, and 2× 7 TB NVMe drives, interconnected via 10 Gbit/s Ethernet, was used for distributed experiments.
An exhaustive search over all reasonable parameter settings of PNS-PDFPN was performed using up to 512 CPU cores.
Due to limited access to the cluster, only the most promising experiments were computed for the 1024-core setting.
To ensure reproducibility of the random sampling used for breaking ties in child selection, random seeds 1, 2, and 3 were fixed.

\subsection{Analysis of Retention and Sharing}

In Figure~\ref{fig_second_experiment_full}, we revisit the comparison between PNS-PNS and PNS-DFPN with varying levels of state retention and key result sharing (Grundy numbers, in the case of impartial games).
We observe that the poor speedup of PNS-PNS without Grundy number retention is primarily due to substantial search overhead, as workers repeatedly recompute results that they have discarded before. 
The utilization of workers is also slightly lower than in the other settings.
We attribute this to increased traffic between workers and the master, as the workers are less able to solve the children of the assigned leaf node~$\ell$ immediately. As a result, they send back a larger number of unproven children, which must then be expanded in the master tree.

\begin{figure*}[t]
\centering
\includegraphics[scale=1]{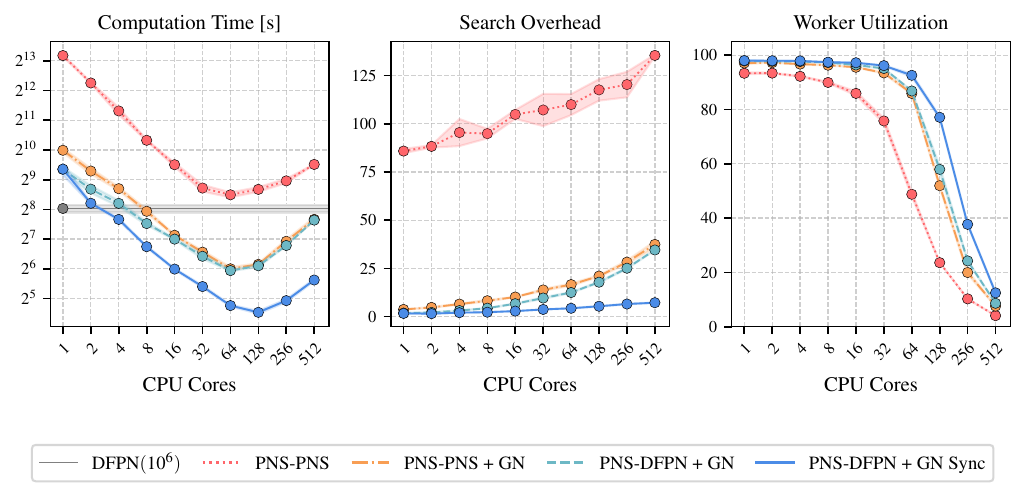}
\caption{Comparison of PNS-PNS and PNS-DFPN, including the impact of retention of Grundy numbers (GN) in workers and their synchronization (GN Sync). All variants are run on the 29-spot position with 100 \texttt{iterations}, 100 \texttt{updates}, no grouping, and no second-level parallelization.}
\centering
\label{fig_second_experiment_full}
\end{figure*}

Retaining Grundy numbers already leads to a substantial reduction in search overhead, resulting in improved speedup.
The same effect is observed when proof and disproof numbers are also retained by using DFPN workers at the second level.
The search overhead is further reduced when all derived Grundy numbers are shared among workers, without compromising worker utilization. 
This highlights the suitability of Grundy numbers as key results to share, as they are inexpensive to distribute and widely reusable.
Thus, PNS-DFPN with Grundy number synchronization achieves superior scaling efficiency compared to its unsynchronized variant, a difference that becomes even more pronounced for the larger position shown in Figure~\ref{fig:main}.

Also, note that beyond a certain point, adding more workers to the algorithm stops being beneficial, as the overhead on the master grows significantly. 
This is evident from the low utilization of the workers.
In the following subsection, we explain how to address this limitation.

\subsection{Optimal Parameter Settings}

We analyze the setting of the following four parameters: \texttt{iterations}, the maximum number of expansions per job; \texttt{updates}, the number of iterations between each update sent to the master; \texttt{grouping}, the number of workers grouped within a single node; and \texttt{threads}, the number of threads assigned to PDFPN per worker.

\paragraph{Updates.}
First, in Figure~\ref{fig_updates_iterations}, we show how setting the number of \texttt{updates} in PNS-PDFPN with GN synchronization, no grouping, and no second-level parallelization affects the scaling.
We observe that a lower number of \texttt{updates} naturally reduces worker utilization, since more frequent synchronization of Grundy numbers and tree updates increase the overhead on the master.
On the other hand, with smaller values of \texttt{updates}, the search overhead increases at a slower rate, as the workers are more synchronized due to more frequent GN synchronization.
Therefore, a reasonable trade-off must be made between worker utilization and search overhead.
Also, notice that increasing \texttt{updates} above 1,000 does not increase the worker utilization, which causes the drop in speedup.
Hence, setting \texttt{updates} to 1,000 almost always achieves the best scaling.

\paragraph{Iterations.}
When the \texttt{updates} parameter is fixed,~the~to\-tal job length can be optimized by adjusting the \texttt{iterations} parameter; see Figure~\ref{fig_updates_iterations}.
With longer jobs, workers switch less frequently between potentially very different parts of the tree, allowing them to better utilize locally stored proof and disproof numbers.
However, a more focused search leads to less exploration of the master tree, which may result in computations in suboptimal parts of the tree.
Therefore, we observe that larger values of \texttt{iterations} benefit PNS-PDFPN when using a large number of workers, as a sufficient number of workers can guarantee sufficient exploration.
The opposite effect is visible with a low number of workers, where short-term jobs achieve better scaling.

\begin{figure*}[!htb]
\centering
\includegraphics[scale=1.02]{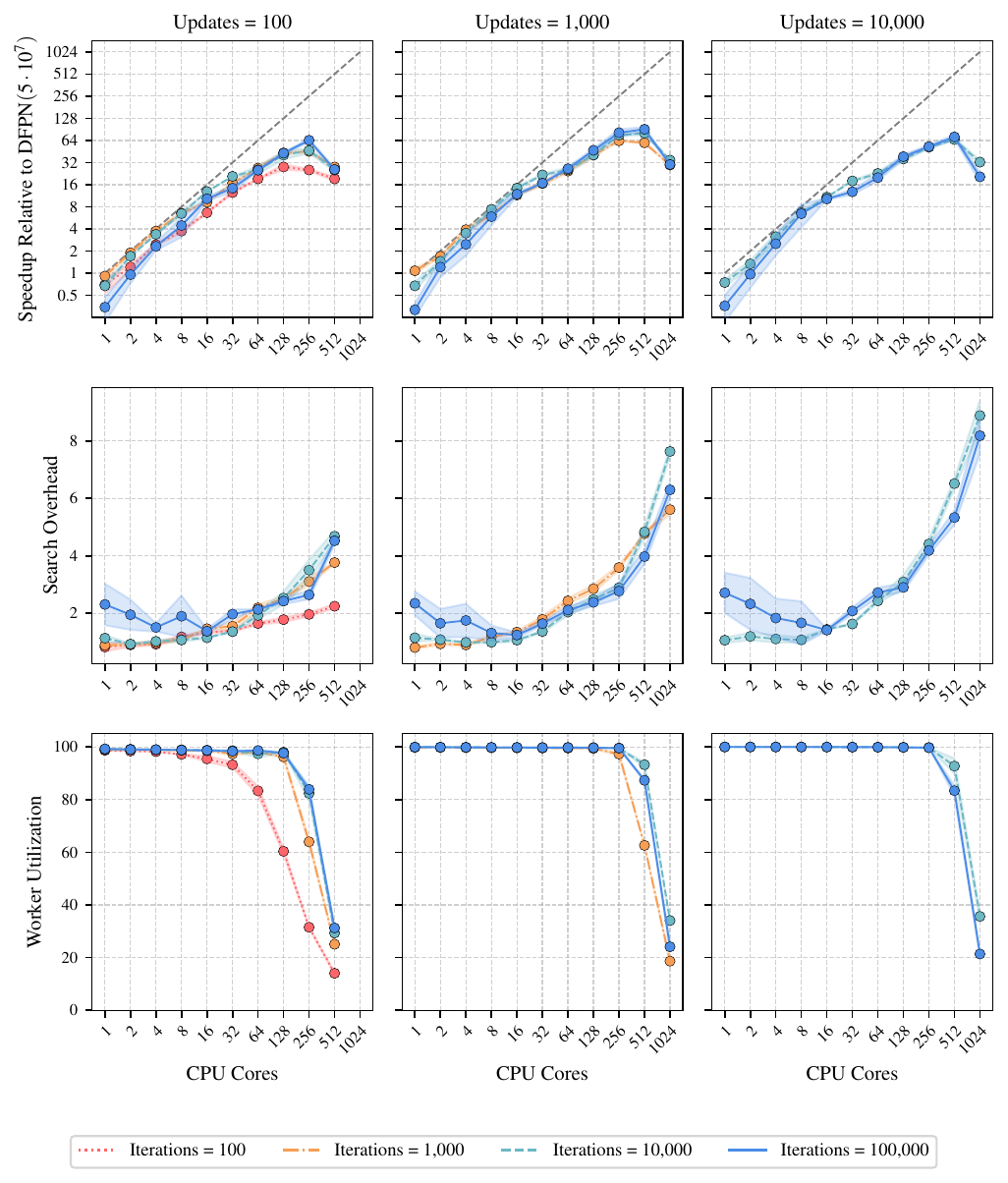}
\caption{Analysis of the impact of \texttt{iterations} and \texttt{updates} settings on the scaling of PNS-PDFPN with GN synchronization, no grouping, and no second-level parallelization.}
\centering
\label{fig_updates_iterations}
\end{figure*}

\paragraph{Grouping.}
As previously observed, beyond a certain po\-int, adding more workers no longer accelerates the algorithm due to excessive overhead on the master.
One way to reduce overhead and improve worker utilization is to group workers located on the same node.
This allows workers to share the local database of key results (Grundy numbers), which reduces memory consumption and allows instant synchronization of key results between workers. 
This slightly decreases the search overhead as shown in Figure~\ref{fig_grouping_updates_iterations}, where we set \texttt{grouping} to 32 (note that the size of \texttt{grouping} is limited by the number of cores of a node).
Moreover, fewer redundant key results derived in parallel by multiple workers are shared, which reduces overhead on the master and consequently improves worker utilization.
Therefore, grouping workers enables PNS-PDFPN to scale efficiently with up to at least 1024 CPU cores.

\begin{figure*}[t]
\centering
\includegraphics[scale=1]{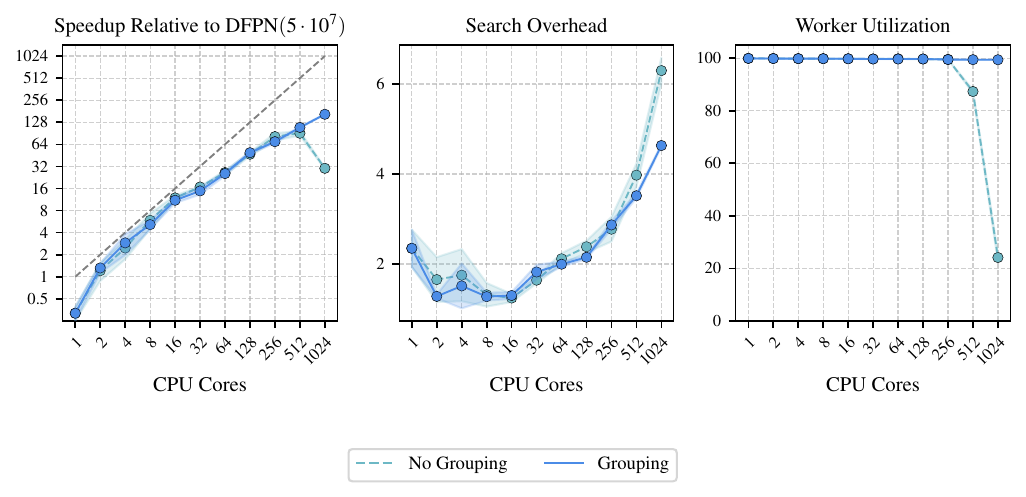}
\caption{Analysis of the impact of \texttt{grouping} settings on the scaling of PNS-PDFPN with GN synchronization and no second-level parallelization.
\texttt{Grouping} is set to its maximum (up to 32) and \texttt{updates} to 1,000.
}
\centering
\label{fig_grouping_updates_iterations}
\end{figure*}

\paragraph{Threads.}
Finally, in Figure~\ref{fig:kaneko_comparison}, we examine how the parameter \texttt{threads}, responsible for second-level parallelization, affects scaling efficiency.
In general, similar to configuring the \texttt{iterations} parameter, stronger second-level parallelization is beneficial only when a larger number of CPU cores is used, as there are already enough workers to ensure sufficient exploration of nodes in the master tree.
This effect is even stronger if the child-ordering heuristic is used to select the child in the case of tied disproof numbers.
In this case, the algorithm can reduce the exploration even more and rely on the heuristic by further increasing the values of \texttt{iterations} and \texttt{threads}, resulting in a more localized search; recall Table~\ref{tab:heuristicEvaluation}, which shows the optimal parameter settings with and without the heuristic.

\begin{figure*}[t]
\centering
\textbf{(a) Without heuristic} \\[0.5em]
\includegraphics[scale=1]{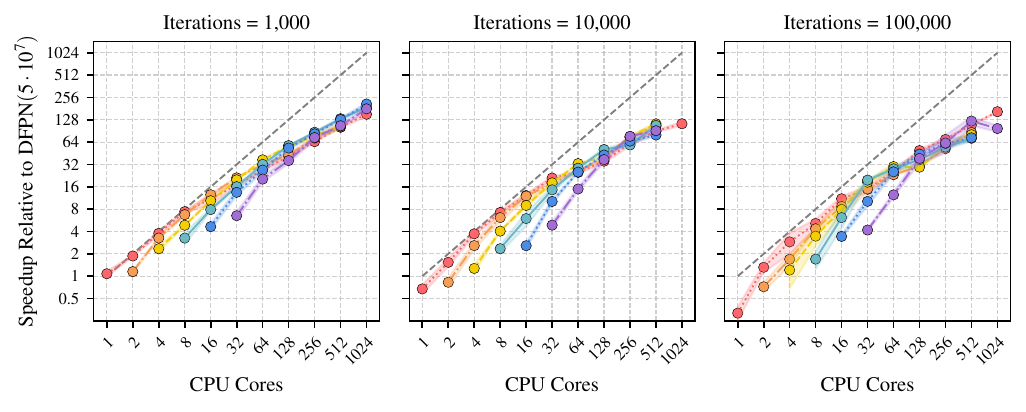}

\vspace{2em}

\textbf{(b) With heuristic} \\[0.5em]
\includegraphics[scale=1]{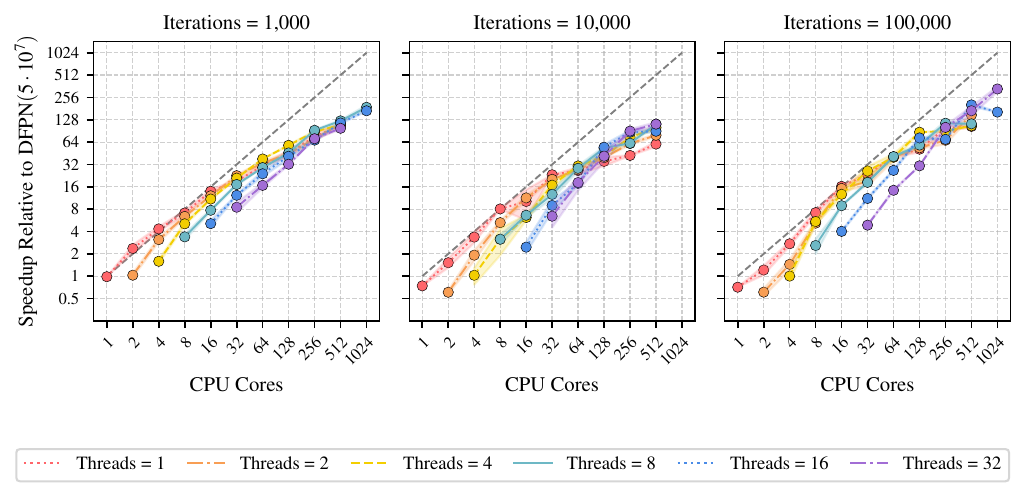}

\caption{Analysis of the impact of \texttt{threads} settings on the scaling of PNS-PDFPN with GN synchronization. Two variants of PNS-PDFPN are evaluated: one without the child-ordering heuristic (a) and one with it (b). \texttt{Grouping} is set to its maximum (up to 32).}
\label{fig:kaneko_comparison}
\end{figure*}

\subsection{Complete Statistics of PNS-PDFPN Scaling}

In Table~\ref{tab:final_full_values}, we present statistics collected for PNS-PDFPN with the child-ordering heuristic after applying the optimal parameter settings, extending Table~\ref{tab:heuristicEvaluation}.
We measured the following seven values on the 47-spot position: the number of nodes in the master tree (M-Nodes), the average number of nodes in workers' transposition tables (W-Nodes), the total number of computed Grundy numbers (T-GN), the total number of iterations (T-It), the worker utilization (Util), the computation time (Time), and the measured speedup.

We observe that worker utilization remains above $99\%$ in all measured instances, up to 1024 CPU cores.
The search overhead increases slightly, from approximately 0.94 with 1 core to 2.50 with 1024 cores.
The speedup continues to grow, reaching $332.97 \pm 26.80$ at 1024 cores.
Measurements were made with respect to DFPN($5 \cdot 10^7$) as the baseline.
Note that reporting the speedup of PNS-PDFPN with respect to PNS-PDFPN with 1 core almost does not change the reported speedup, since both PNS-PDFPN and DFPN have almost identical running times.

We believe that the algorithm can scale well even beyond 1024 cores, as PNS-PDFPN seems to scale effectively as long as second-level parallelization can be increased.
This may be feasible, given that PDFPN scales well on up to 64 cores.

\begin{table*}[!htb]
    \centering
    \begin{tabular}{rrrrrrrrrr}
        \toprule
        Cores & \multicolumn{2}{c}{Parameters} & \multicolumn{7}{c}{PNS-PDFPN + Heuristic} \\
        \cmidrule(r){2-3}
        \cmidrule(lr){4-10}
        & \multicolumn{1}{c}{It} & \multicolumn{1}{c}{Th} & \multicolumn{1}{c}{M-Nodes} & \multicolumn{1}{c}{W-Nodes} & \multicolumn{1}{c}{T-GN} & \multicolumn{1}{c}{T-It} & \multicolumn{1}{c}{Util} & \multicolumn{1}{c}{Time [s]} & \multicolumn{1}{c}{Speedup} \\
        \cmidrule(r){1-10}
    DFPN&  \multicolumn{1}{c}{---} &\multicolumn{1}{c}{---}&15M\sem{1.6M}  &\multicolumn{1}{c}{---}   &268k\sem{30k\;}  &17.2M\sem{1.9M}   &\multicolumn{1}{c}{---}&17.6k\sem{2.5k}  &\multicolumn{1}{c}{---}   \\
   \midrule
   1&  1k &1&609k\sem{0.0\;}  &14.3M\sem{0.0\;\;\;\,}   &199k\sem{0.0\;\;}  &16.3M\sem{0.0\;\;\;}   &99.8\sem{0.0}&17.9k\sem{4.29\,}  &{\bfseries 0.99}\sem{0.00}   \\
   2&  1k &1&500k\sem{56k}&6.5M\sem{665k\;} &159k\sem{15k\;}&14.4M\sem{1.5M}  &99.7\sem{0.0}&7.67k\sem{872\,} &{\bfseries 2.36}\sem{0.28}   \\
   4&  1k &1&540k\sem{40k}&3.7M\sem{303k\;} &152k\sem{13k\;}&16.2M\sem{1.3M}  &99.7\sem{0.0}&4.11k\sem{285\,} &{\bfseries 4.34}\sem{0.32}   \\
   8&100k &1&64k\sem{770} &1.8M\sem{11.8k}  &130k\sem{4.0k} &16.0M\sem{0.2M}&99.7\sem{0.0}&2.19k\sem{24.9}  &{\bfseries 8.06}\sem{0.09}   \\   
  16&100k &2&6.4k\sem{223}  &0.9M\sem{35.6k}&150k\sem{6.1k} &17.6M\sem{0.7M}&99.7\sem{0.0}&1.09k\sem{34.9}  &{\bfseries 16.22}\sem{0.53}  \\
  32&100k &1&1.9k\sem{294}  &1.5M\sem{231k\;} &93.8k\sem{17k\;} &16.3M\sem{2.5M}  &99.7\sem{0.0}&707\sem{112\;}&{\bfseries 26.13}\sem{3.87}  \\
  64&100k &2&1.5k\sem{0.3\;}    &1.5M\sem{57.5k}  &90.7k\sem{3.2k}  &18.5M\sem{0.7M}&99.7\sem{0.0}&432\sem{12.2} &{\bfseries 40.86}\sem{1.13}  \\
 128&100k &4&3.3k\sem{1.2\;}    &0.5M\sem{5.63k} &90.2k\sem{1.3k}  &20.9M\sem{0.2M}&99.7\sem{0.0}&202\sem{1.90}  &{\bfseries 87.02}\sem{0.81}  \\
 256&100k &8&3.3k\sem{12\;\;}   &0.6M\sem{41.2k}&122k\sem{6.5k} &30.6M\sem{2.0M}  &99.6\sem{0.0}&155\sem{11.1} &{\bfseries 115.11}\sem{8.88} \\
 512&100k &16&3.3k\sem{1.2\;}    &0.6M\sem{22.8k}&112k\sem{6.9k} &32.9M\sem{1.3M}  &99.5\sem{0.0}&87.3\sem{4.43}   &{\bfseries 203.05}\sem{9.82} \\
1024&100k &32&3.2k\sem{0.3\;}    &0.6M\sem{49.1k}&119k\sem{13k\;}&43.0M\sem{3.2M}  &99.2\sem{0.0}&53.7\sem{4.69}   &{\bfseries 332.97}\sem{26.8}\\
        \bottomrule
    \end{tabular}
    \caption{Statistics measured for PNS-PDFPN with the child-ordering heuristic after applying the optimal parameter settings. The measured values are the number of nodes in the master tree (M-Nodes), the average number of nodes in workers' transposition tables (W-Nodes), the total number of computed Grundy numbers (T-GN), the total number of iterations (T-It), the worker utilization (Util), the computation time (Time), and the measured speedup.
    We also include the value of the parameters \texttt{iterations} (It) and \texttt{threads} (Th) and measurement for DFPN($5 \cdot 10^7$).
    \texttt{Grouping} set to its maximum and \texttt{updates} to 1,000 are optimal across all settings.
    Workers' transposition table capacity is $5 \cdot 10^7$.
    Measured on the 47-spot position.}
    \label{tab:final_full_values}
\end{table*}

\end{document}